\documentclass{article}

\usepackage{microtype}
\usepackage{graphicx}
\usepackage{subfigure}
\usepackage{booktabs} 

\usepackage{hyperref}

\usepackage[accepted]{icml2024}

\usepackage{amsmath}
\usepackage{amssymb}
\usepackage{mathtools}
\usepackage{amsthm}

\usepackage[capitalize,noabbrev]{cleveref}

\theoremstyle{plain}
\newtheorem{theorem}{Theorem}[section]
\newtheorem{proposition}[theorem]{Proposition}
\newtheorem{lemma}[theorem]{Lemma}

\theoremstyle{definition}

\theoremstyle{remark}

\usepackage[textsize=tiny]{todonotes}

\usepackage{amsmath}
\usepackage{amssymb}
\usepackage{amsthm}
\usepackage{mathrsfs}
\usepackage{mathabx}
\usepackage{dsfont}
\usepackage{mathtools}
\usepackage{stmaryrd}

\usepackage{tikz}
\usetikzlibrary{fit,arrows,positioning}

\usepackage{dirtytalk}

\usepackage{xspace}

\usepackage{hyperref}

\DeclareMathOperator*{\Argmin}{arg\,min}
\DeclareMathOperator{\Exp}{exp}
\DeclareMathOperator{\Expec}{\mathbb{E}}
\DeclareMathOperator{\Proba}{\mathbb{P}}
\DeclareMathOperator{\Var}{Var}

\newcommand{\abar}{\overline{a}}
\newcommand{\abs}[1]{\left\lvert#1\right\rvert}
\newcommand{\bbar}{\overline{b}}
\newcommand{\betahat}{\widehat{\beta}}%
\newcommand{\condexpec}[2]{\mathbb{E}\left[#1\middle|#2\right]}
\newcommand{\condexpecunder}[3]{\mathbb{E}_{#3}\left[#1\middle|#2\right]}

\newcommand{\Diff}{\mathrm{d}}
\newcommand{\defeq}{\vcentcolon =}
\newcommand{\empvar}[1]{\widehat{\Var}\left(#1\right)}
\renewcommand{\exp}[1]{\Exp\left(#1\right)}
\newcommand{\expec}[1]{\Expec\left[#1\right]}
\newcommand{\Indic}{\mathds{1}}
\newcommand{\indic}[1]{\Indic_{#1}}
\newcommand{\Integers}{\mathbb{Z}}

\newcommand{\norm}[1]{\left\lVert#1\right\rVert}
\newcommand{\probaunder}[2]{\Proba_{#2}\left(#1\right)}
\newcommand{\condprobaunder}[3]{\Proba_{#2}\left(#1\mid #3\right)}
\newcommand{\condvarunder}[3]{\Var_{#2}(#1\mid #3)}
\newcommand{\Reals}{\mathbb{R}}

\newcommand{\var}[1]{\Var\left(#1\right)}

\newcommand{\bigo}[1]{\mathcal{O}\left(#1\right)}

\renewcommand{\epsilon}{\varepsilon}

\theoremstyle{plain}

\newcommand{\betainf}{\beta^\infty}
\newcommand{\datt}{d_\text{att}}

\newcommand{\dout}{d_\text{out}}
\newcommand{\tmax}{T_{\max}}

\newcommand{\vtilde}{\tilde{v}}
\newcommand{\wordemb}{W_e}
\newcommand{\Wk}{W_k}
\newcommand{\Wl}{W_{\ell}}
\newcommand{\Wp}{W_p}
\newcommand{\Wq}{W_q}
\newcommand{\Wv}{W_v}

\makeatletter
\def\hlinewd#1{%
	\noalign{\ifnum0=`}\fi\hrule \@height #1 %
	\futurelet\reserved@a\@xhline}
\makeatother

\makeatletter
\def\th@plain{%
	\thm@notefont{}
	\itshape 
}
\def\th@definition{%
	\thm@notefont{}
	\normalfont 
}
\makeatother

\newcommand{\explainer}[1]{\text{#1}\xspace}
\newcommand{\attavg}{\explainer{$\alpha$-avg}}
\newcommand{\attmax}{\explainer{$\alpha$-max}}
\newcommand{\lime}{\explainer{lime}}
\newcommand{\gradavg}{\explainer{G-avg}}
\newcommand{\gradnorm}[1]{\explainer{G-l#1}}
\newcommand{\gradtimesi}{\explainer{G$\times$I}}

\icmltitlerunning{Attention Meets Post-hoc Interpretability}

\begin{document}

\twocolumn[
\icmltitle{Attention Meets Post-hoc Interpretability: \\ A Mathematical Perspective}

\icmlsetsymbol{equal}{*}

\begin{icmlauthorlist}
\icmlauthor{Gianluigi Lopardo}{1,2}
\icmlauthor{Frederic Precioso}{1,3}
\icmlauthor{Damien Garreau}{4}
\end{icmlauthorlist}

\icmlaffiliation{1}{Universit\'e C\^ote d'Azur, Inria, CNRS}
\icmlaffiliation{2}{LJAD}
\icmlaffiliation{3}{I3S}
\icmlaffiliation{4}{Julius-Maximilians-Universit\"at W\"urzburg}

\icmlcorrespondingauthor{Gianluigi Lopardo}{glopardo@unice.fr}

\icmlkeywords{Explainable AI, transformers, attention}

\vskip 0.3in
]

\printAffiliationsAndNotice{ }

\begin{abstract}
Attention-based architectures, in particular transformers, are at the heart of a technological revolution. 
Interestingly, in addition to helping obtain state-of-the-art results on a wide range of applications, the attention mechanism intrinsically provides meaningful insights on the internal behavior of the model. 
Can these insights be used as explanations? 
Debate rages on. 
In this paper, we mathematically study a simple attention-based architecture and pinpoint the differences between post-hoc and attention-based explanations. 
We show that they provide quite different results, and that, despite their limitations, post-hoc methods are capable of capturing more useful insights than merely examining the attention weights. 
\end{abstract}

\section{Introduction}
%
\begin{figure}[t]
\centering
\addtolength{\tabcolsep}{-0.3cm}
\begin{tabular}{l l}
\textbf{\attavg:} & \includegraphics[scale=0.32]{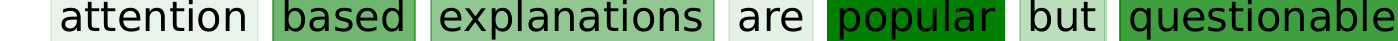} \\
\textbf{\attmax:} & \includegraphics[scale=0.32]{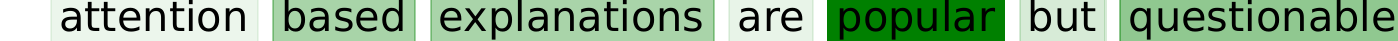} \\
\textbf{\lime:} & \includegraphics[scale=0.32]{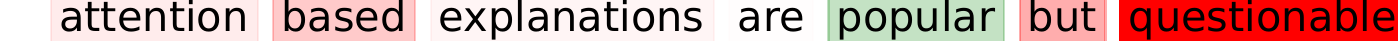} \\
\textbf{\gradavg:} & \includegraphics[scale=0.32]{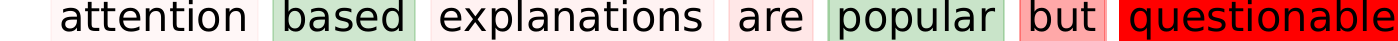} \\
\textbf{\gradnorm{1}:} & \includegraphics[scale=0.32]{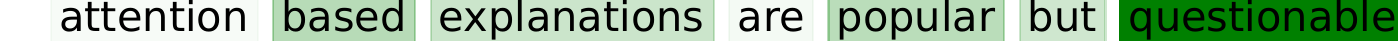} \\
\textbf{\gradnorm{2}:} & \includegraphics[scale=0.32]{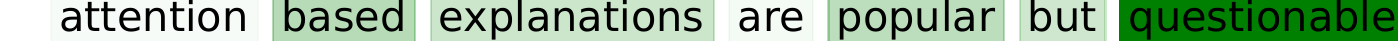} \\
\textbf{\gradtimesi:}  & \includegraphics[scale=0.32]{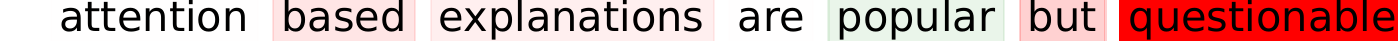} \\
\end{tabular}
\caption{\label{fig:attn-xai}Different explainers can produce very different explanations. 
Here, the \emph{attention} mean (\textbf{\attavg}) and maximum (\textbf{\attmax}) over the heads, \emph{LIME} (\textbf{\lime}), the \emph{gradient} mean (\textbf{\gradavg}), $L^1$ norm (\textbf{\gradnorm{1}}), and $L^2$ norm (\textbf{\gradnorm{2}}), with respect to the tokens, and \emph{Gradient times Input} (\textbf{\gradtimesi}) are employed to interpret the prediction of a sentiment-analysis model. 
Words with positive (respectively, negative) weights are highlighted in green (respectively, red), with intensity proportional to their weight. 
In the example, all the explainers identify the word \emph{questionable} as highly significant, while only \lime, and \gradtimesi highlight a negative contribution. 
Interestingly, \attavg and \attmax identify the word \emph{popular} as the most important word in absolute terms, in disagreement with the all others. 
}
\end{figure}

The attention mechanism, introduced by \citet{bahdanau2015neural}, revolutionized neural networks by enabling models to focus dynamically on different parts of input sequences, enhancing their ability to capture long-range dependencies. 
This innovation laid the groundwork for various deep learning models. 
The Transformer architecture, introduced by \citet{vaswani_et_al_2017}, is a notable application of the attention mechanism. 
Initially designed for natural language processing (NLP) tasks, the Transformer eliminated the need for recurrent neural networks and convolutional layers, relying solely on attention mechanisms. 
The Transformer has since become the state-of-the-art in numerous machine learning domains due to its flexibility, performance, and ability to model complex relationships in data. 
Its innovative design and significant improvements in training efficiency have paved the way for the development of advanced models such as BERT \citep{devlin2019bert}, and GPT-3 \citep{brown_et_al_2020}, which have revolutionized NLP. 

As a by-product of the attention mechanism, weights corresponding to the per-token attention at a given layer can easily be extracted from the model.   
One is tempted to use these weights as explanations for the model’s predictions, and many researchers have indeed done so \citep{chefer2021transformer, mylonas2023attention}. 
However, the use of attention mechanisms for explainability has been questioned in the literature. 
\citet{jain2019attention} notably critique its clarity, questioning the relationship between attention weights and model output. 
Conversely, \citet{wiegreffe2019attention} argue that attention mechanisms remain useful for interpretability, without specifically addressing \citet{jain2019attention}'s requirements. 
These works have sparked an intriguing debate in the literature, which we develop in Section~\ref{sec:related-work}. 
In our opinion, neither stance has provided a solid theoretical foundation to support their respective claims. 

In this paper, we propose a mathematical analysis of an attention-based model and the associated explanations, trying to shed light on the respective merits of each approach not merely relying on experimental validation but truly looking at the connection between given explanations and the model's structure and parameters. 
Our analysis centers on a single-layer multi-head network, detailed in Section~\ref{sec:the-model}. 
This is a simplified variant of the transformer architecture proposed by \citet{vaswani_et_al_2017}, tailored for a binary prediction task. 
Note that the binary classification restriction is illustrative and without loss of generality; the same results hold for multi-label predictions, when examining a specific class of interest. 
Additionally, while we focus on text classification tasks, we analyze token-level explanations, which could also be pixels in the context of Vision Transformers. 

Specifically, we analyze the connections between attention-based and established post-hoc explanations. 
These methods include gradient-based, such as \emph{Gradient} \citep{li2016visualizing}, \emph{Gradient$\times$Input} \citep{denil2014extraction}, and perturbation-based approaches, such as LIME \citep{ribeiro_et_al_2016}. 

We demonstrate that perturbation-based and gradient-based methods provide more insightful explanations than solely examining attention weights in Transformer models. 
We particularly concur with \citet{bastings2020elephant} that attention weights, while useful for input token weighting, can be misleading as model predictions’ explanations, and advocate for post-hoc approaches.

\paragraph{Summary of the paper.}
In Section \ref{sec:related-work}, we discuss the relevant literature, in particular focusing on the debate around attention-based explanations. 
We describe the model that we study in Section~\ref{sec:the-model}. 
In Section~\ref{sec:attn-exp} we specifically discuss attention-based explanations. 
In Section~\ref{sec:gradient} (resp. Section~\ref{sec:lime}), we derive expressions for the gradient (resp. LIME) explanations associated to our model. 
These expressions (Theorems~\ref{th:gradient-meets-attention} and~\ref{th:lime-meets-attention}) are \emph{explicit} with respect to the model parameters and the input document, thus allowing us to pinpoint exactly the differences between these approaches. 
In Section \ref{sec:limitation}, we discuss the main limitations of our work, including the theoretical assumptions underlying the model under examination.
We draw our conclusion in Section~\ref{sec:conclusion}. 
All our theoretical claims are supported by mathematical proofs and empirical validation, detailed in the Appendix.  
The code for the model and the experiments are available at \url{https://github.com/gianluigilopardo/attention_meets_xai}. 


\subsection{Related work}
\label{sec:related-work}
The attention mechanism, pioneered by \citet{bahdanau2015neural}, enhanced neural networks’ ability to focus on different parts of input sequences. 
Various forms of attention mechanisms exist, each characterized by distinct methods of query generation. 
A key differentiation lies in the computation of attention weights. 
Two methods are additive attention, as originally proposed by \citet{bahdanau2015neural}, and \textbf{scaled dot-product attention} (on which we focus our study), introduced by \citet{vaswani_et_al_2017}. 
Despite their differences, these two forms are theoretically similar \citep{vaswani_et_al_2017} and have been found to yield comparable results \citep{jain2019attention}. 
This innovation paved the way for various deep learning models, including the Transformer architecture by \citet{vaswani_et_al_2017}. 

In essence, self-attention quantifies how much each token in a sequence is related with every other token. 
This relation is represented as attention weights, indicating the model’s focus on different parts of the input. 
Thus, it is tempting to use these attention weights as explanations for the model’s predictions. 
They provide a seemingly intuitive way to understand what the model is \emph{paying attention to} when making a decision. 
Indeed, there are several methods to generate attention-based explanations. 
We delve into a discussion of these various methods in Section~\ref{sec:attn-exp} of the paper. 
While attention weights offer valuable insights into the model’s behavior, the use of attention mechanisms for explainability has been met with skepticism in the literature, generating an ongoing debate, that we summarize in the following.  

\paragraph{The debate.}
A significant critique is offered by \citet{jain2019attention}, questioning the relationship between attention weights and model output. 
They argue, based on experiments across various NLP tasks, that attention weights do not provide meaningful explanations. 
In particular, \citet{jain2019attention} proposes two properties that should hold if attention provides faithful explanations: (i) attention weights should correlate with feature importance measures (gradient-based measures and leave-one-out), and (ii) alternative (or counterfactual) attention weight configurations should yield corresponding changes in prediction. 
However, their experiments, suggest that these properties do not hold, leading them to conclude that attention weights are not suitable for interpretability. 

On the other hand, \citet{jain2019attention} has several limitations, first highlighted by \citet{wiegreffe2019attention}. 
Experimentally, \citet{wiegreffe2019attention} conclude that \say{prior work does not disprove the usefulness of the attention mechanism for interpretability.} 
They do not specifically address the claims presented by \citet{jain2019attention}, but they critique the experimental design proposed for point (ii), while somewhat agreeing with the first observation and corresponding experimental setup. 

Specifically, \citet{wiegreffe2019attention} introduce an end-to-end model training approach for finding adversarial attention weights. 
This approach ensures that the new, adversarial attention weights are plausible and consistent with the model. 
This is in contrast to the approach taken by \citet{jain2019attention}, where only the attention scores were changed, disrupting the model’s training. 
Furthermore, \citet{wiegreffe2019attention} argue against the exclusive explanation that \say{attention is an explanation, not the explanation.} 

\citet{serrano2019attention} also scrutinize the use of attention for interpretability. 
They manipulate attention weights in pre-trained text classification models and analyze the impact on predictions. 
Their conclusion is that attention provides a noisy prediction of the input tokens’ overall importance to a model, but it is not a reliable indicator. 

More recently, \citet{bibal2022attention} provide an overview of the debate on whether attention serves as an explanation, focusing on literature that builds on the works of \citet{jain2019attention} and \citet{wiegreffe2019attention}. 
\citet{bibal2022attention} argue that the applicability of attention as an explanation heavily depends on the specific NLP task. 
For instance, \citet{clark2019does} demonstrate that BERT’s attention mechanism can provide reliable explanations for syntax-related tasks like part-of-speech tagging. 
Similar results are presented by \citet{vig2019analyzing} for GPT-2. 
In general, syntactic knowledge appears to be encoded across various attention heads and layers. 
\citet{galassi2020attention} show that attention in transformers focuses on syntactic structures, making it suitable for global explanation. 

\citet{brunner2020identifiability} theoretically demonstrate that attention weights can be decomposed into two parts, with the \emph{effective attention} part focusing on the effective input without being biased by its representation. 
This work is further expanded by \citet{kobayashi2020attention} and \citet{sun2021effective}, who conduct a more in-depth evaluation. 
They find that alternative attention distributions obtained through adversarial training perform poorly, suggesting that the attention mechanism of RNNs indeed learns something useful. 
This finding contradicts the claim by \citet{jain2019attention} that attention weights do not provide meaningful explanations. 

It is important to note that there is currently no definitive theoretical support for either side of the debate on whether attention serves as an explanation. 
The positions presented by both \citet{jain2019attention} and \citet{wiegreffe2019attention} are primarily based on empirical experiments. 
The subsequent debate provided valuable insights and provoked thoughtful discussion, but did not conclusively prove or disprove the interpretability of attention mechanisms. 

However, some recent works have investigated the role of attention through mathematical examination on specific tasks, in a similar fashion to our work. 
\citet{wen2024transformers} examines transformer interpretability by analyzing the model's weight matrices and attention patterns in the context of learning a \emph{Dyck language} \citep{schutzenberger1963context}. 
The authors demonstrate that vastly different solutions can be reached via standard training, cautioning against making interpretability claims based on inspecting individual components of the model. 
In particular, the attention pattern of a single layer can be \say{nearly randomized} and still achieve high accuracy. 
In the same spirit, \citet{li2023transformers} provides a mechanistic understanding of how transformers learn semantic structure, through mathematical analysis and experiments on Wikipedia and LDA-generated \citep{blei2003latent} data. 
The study shows that both the embedding and self-attention layers can encode topical structures. 
In essence, even when the attention score is set to be uniform, the transformer can achieve a near optimal loss, as other parts of the model compensate for it. 
Finally, \citet{cui2024phase} demonstrate that for a simple counting task (the \emph{histogram task} defined in \citet{weiss2021thinking}), the loss landscape of a transformer with a dot-product attention layer and positional encodings reveals two distinct solutions: one with an attention matrix largely independent of the input tokens, and another that varies significantly based on the tokens and their semantic content. 
Ultimately, these works demonstrate that there is no evidence to assume that attention scores capture the core information underlying a transformer's predictions.

\paragraph{Post-hoc interpretability.}
Post-hoc interpretability refers to the process of explaining a model’s predictions after it has been trained \citep{linardatos2021explainable, bodria2023benchmarking}. 
Among these techniques are gradient-based explanations \citep{li2016visualizing,poerner2018evaluating,arras2019evaluating,atanasova2020diagnostic,denil2014extraction,sundararajan2017axiomatic}, that leverage the gradients of the model’s output with respect to its input. 
Conversely, perturbation-based methods like LIME \citep{ribeiro_et_al_2016}, SHAP \citep{lundberg2017unified}, Anchors \citep{ribeiro2018anchors}, or FRED \citep{lopardo2023faithful}, examine the alterations in the model’s output in response to changes in its input. 

While post-hoc explanations offer valuable insights into machine learning models, they are not flawless. 
Their complexity can lead to inaccuracies, with sampling mechanisms potentially causing out-of-distribution issues and adversarial attacks \citep{hase2021out, slack2020fooling}. 
Some popular explainers have also been found lacking in soundness \citep{marques2022delivering}. 
Thus, their use and interpretation necessitate careful scrutiny and continued research. 
\citet{mardaoui_garreau_2021} and \citet{lopardo2022sea} respectively propose deep theoretical analysis of LIME and Anchors for NLP models. 
In Section~\ref{sec:lime}, we leverage \citet{mardaoui_garreau_2021}'s work to compute LIME coefficients for our model.

\paragraph{Attention meets post-hoc interpretability.}
Existing literature explores the differences between post-hoc and attention-based explanations, employing diverse methodologies and drawing varied conclusions. 
\citet{ethayarajh2021attention} formally establish that attention weights are not Shapley values \citep{shapley1953value}, but attention flows (a post-processed version of attention weights) are, at least at the layerwise level. 
\citet{thorne2019generating} conduct a comparative analysis between post-hoc and attention-based methods. 
They select key features according to each explainer, subsequently using these features to make predictions and evaluate their accuracy. 
Their findings indicate that post-hoc methods like LIME and Anchors yield more accurate explanations than attention-based ones when implemented on an LSTM \citep{sak2014long} for natural language inference. 

\citet{neely2021order} evaluate the ``agreement as evaluation'' paradigm, comparing various explanation methods on Bi-LSTM \citep{huang2015bidirectional} and Distil-BERT \citep{sanh2019distilbert} models. 
They conclude that consistency between different explainers should not be a criterion for evaluation unless a proper ground truth is available. 
This contradicts the agreement between \citet{jain2019attention} and \citet{wiegreffe2019attention}. 
They also highlight theoretical limitations of state-of-the-art explainers and suggest using solid diagnostic tools like those proposed by \citet{atanasova2020diagnostic}.

\citet{neely_song_2022} build on \citet{neely2021order}'s work, finding a lack of correlation among explanation methods, particularly in complex settings. 
They question the existence of an \emph{ideal} explanation and the use of the \emph{agreement as evaluation} paradigm for comparison, by demonstrating that similar explanations may not yield correlated rankings.



\section{Attention-based classifier}
\label{sec:the-model}

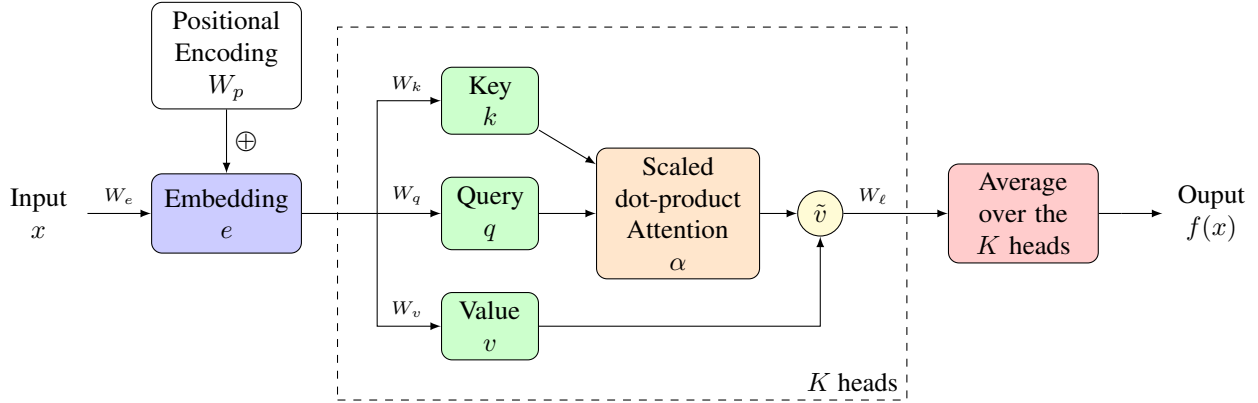
\begin{figure*}
    \centering
    \begin{tikzpicture}[
    node distance = 2.5cm,
    auto,
    emb_block/.style = {rectangle, draw, fill=blue!20, text width=5em, text centered, rounded corners, minimum height=3em},
    pos_block/.style = {rectangle, draw, 
    text width=5em, text centered, rounded corners, minimum height=3em},
    kvq_block/.style = {rectangle, draw, fill=green!20, text width=3em, text centered, rounded corners, minimum height=2em},
    attn_block/.style = {rectangle, draw, fill=orange!20, text width=5.5em, text centered, rounded corners, minimum height=3em},
    avg_block/.style = {rectangle, draw, fill=red!20, text width=5em, text centered, rounded corners, minimum height=3em},
    line/.style = {draw, -latex},
    line_no_arrow/.style = {draw},
    big_block/.style = {rectangle, draw, dashed,
    inner sep=1.5em},
    circle_block/.style = {circle, draw, fill=yellow!20, minimum size=0.6cm, inner sep=0pt}
]

\node (input) [text width=3em, text centered] {Input \\$x$};
\node [emb_block, right of=input] (emb) {Embedding $e$};
\node [pos_block, above of=emb, yshift=-0.4cm] (pe) {Positional Encoding $\Wp$};

\coordinate (A) at ([xshift=1cm]emb.east);

\node [kvq_block, right of=A, xshift=-1cm] (q) {Query \\$q$};
\node [kvq_block, above of=q, yshift=-1cm] (k) {Key \\$k$};
\node [kvq_block, below of=q, yshift=1cm] (v) {Value \\$v$};
\node [attn_block, right of=q] (alpha) {Scaled \\dot-product Attention \\$\alpha$};

\coordinate (B) at ([xshift=0.8cm]alpha.east);

\node [avg_block, right of=B, xshift=0.2cm] (l) {Average over the $K$ heads};
\node [right of=l,text width=3em, text centered] (f) {Ouput \\$f(x)$};

\coordinate (C) at ([xshift=0.3cm]l.east);

\path [line] (input) -- (emb) node[midway, font=\scriptsize] {$W_e$};
\path [line] (pe) -- node[midway,right] {$\oplus$} (emb);
\path [line_no_arrow] (emb) -- (A);
\path [line] (A) |- (k) node[midway, above, xshift=0.4cm, font=\scriptsize] {$\Wk$};
\path [line] (A) |- (q) node[midway, above, xshift=0.4cm, font=\scriptsize] {$\Wq$};
\path [line] (A) |- (v) node[midway, above, xshift=0.4cm, font=\scriptsize] {$\Wv$};;
\path [line] (k) -- (alpha);
\path [line] (q) -- (alpha);

\node [circle_block] (v_tilde) at (B) {$\vtilde$};
\node [right of=v_tilde, xshift=-2cm] (Wl) {};

\path [line] (alpha) -- (v_tilde);
\path [line] (v_tilde) -- (l) node[midway, xshift=-0.3cm, font=\scriptsize] {$\Wl$};
\path [line_no_arrow] (l) -- (C);
\path [line] (C) |- (f);
\path [line] (v) -| (v_tilde);

\node [big_block, fit=(A) (k) (q) (v) (alpha) (v_tilde) (Wl), label={[anchor=south east]south east:$K$ heads}] {};

\end{tikzpicture}
    \caption{\label{fig:model-schema}Illustration of the architecture of the model defined in Section~\ref{sec:the-model}. 
    The input text, denoted as $x \in [D]^T$, is transformed into an embedding $e \in \Reals^{T \times d_e}$ by summing word embeddings and positional encodings as in Eq.~\eqref{eq:def-embedding}. 
    For each of the $K$ heads, the key $k \in\Reals^{T \times \datt}$, query $q \in\Reals^{T \times \datt}$, and value $v \in\Reals^{T \times \dout}$ matrices are computed by applying linear transformations to $e$ using $\Wk,\Wq \in\Reals^{\datt\times d_e}$, and $\Wv\in\Reals^{\dout\times d_e}$, respectively. The attention weights $\alpha \in \Reals^T$ are then computed as the softmax of the scaled dot-product of $k$ and $q$, as per Eq.~\eqref{eq:def-attention}. 
    Then the intermediary output $\vtilde \in \Reals^{\dout}$ is computed are the average of the values $v$ weighted by the attention $\alpha$.
    Each head outputs the linear transformation $\Wl \in \Reals^{1\times \dout}$ of the $\vtilde$ associated with the query corresponding to the \texttt{[CLS]} token. 
    The final prediction $f(x)$ of the model is the average of the outputs across all heads.
        }
\end{figure*}

In this section, we introduce the attention-based architecture that we study throughout the paper. 
We follow \citet{phuong_hutter_2022} in our presentation and notation. 
For any integer $n$, we set $[n]\defeq \{1,\ldots,n\}$. 


\subsection{General description}

In this paper, we consider a set of tokens belonging to a dictionary identified with $[D]$. 
A document $\xi$ is an ordered sequence of tokens $\xi_1,\ldots,\xi_T$. 
We say that~$T$ is the length of the document. 
Without loss of generality, we can assume that the $d$ unique tokens of $\xi$ are the first $d$ elements of $[D]$. 

Our model $f$ is a single-layer, multi-head, attention-based network, followed by a linear layer. 
More formally: 
\begin{equation}
\label{eq:def-model}
f(x) \defeq \frac{1}{K} \sum_{i=1}^K f_i(x) 
= \frac{1}{K} \sum_{i=1}^K W_{\ell}^{(i)} \vtilde^{(i)}(x) 
\, ,
\end{equation}
where $f_i\defeq \Wl^{(i)} \vtilde^{(i)}\in \Reals^{\dout}$, with $\Wl^{(i)}\in\Reals^{1\times \dout}$ the part of the final linear layer associated to head $i$, and for $i\in [K]$, $\vtilde^{(i)}(x)$ is the output of an individual head defined by Eq.~\eqref{eq:def-output}. 
The value of $f$ is used for classification; keeping in mind the sentiment analysis task described in the introduction, document $\xi$ is classified as positive if $f(\xi) > 0$.


\subsection{Attention}
\label{sec:attention}

We now describe mathematically the self-attention mechanism at play in each head. 
Formally, we describe $f_i$ for a given $i$, and thus temporarily drop the $i$ index. 

\paragraph{Token embedding.}
First, for each $t\in [T]$, the token $\xi_t=j$ is embedded as 
\begin{equation}
\label{eq:def-embedding}
e_t \defeq \left(\wordemb\right)_{:,j} + \Wp(t) \in \Reals^{d_e}
\, ,
\end{equation}
where $\wordemb \in\Reals^{d_e\times D}$ is a matrix containing the embeddings of all tokens, where as $\Wp : \Integers \to \Reals^{d_e}$ is a deterministic mapping often called the \emph{positional embedding}. 
It is common to set 
\begin{equation}
\label{eq:positional_embedding}
\begin{cases}
\Wp(t)_{2i} &= \cos(t/\tmax^{2i/d_e}) \\
\Wp(t)_{2i-1} &= \sin(t/\tmax^{2i/d_e})
\, ,
\end{cases}
\end{equation}
with $\tmax$ is the maximal document size. 
For all $T < t \leq \tmax$, the embedding of the fictitious token value is set to an arbitrary $h\in\Reals^{d_e}$, while the positional embedding remains the same. 
In other words, 
\begin{equation}
\label{eq:def-embedding-padding}
\forall T < t \leq \tmax, \quad  e_t \defeq h + \Wp(t) \in \Reals^{d_e}
\, . 
\end{equation}
If $T > \tmax$, the last tokens are ignored and the input document is effectively discarded. 
We assume that the embedding matrices are \emph{shared} between the $K$ heads, but we want to emphasize that our analysis is easily amenable to different embedding matrices for each individual head. 

\paragraph{Keys, queries, values.}
Next, these embeddings are mapped to \emph{key}, \emph{query}, and \emph{values} vectors, defined respectively as
\begin{equation}
\label{eq:def-key}
k_t \defeq \Wk e_t + b_k \in\Reals^{\datt}
\, ,
\end{equation}

\begin{equation}
\label{eq:def-query}
q_t \defeq \Wq e_t + b_q \in\Reals^{\datt}
\, ,
\end{equation}
and
\begin{equation}
\label{eq:def-value}
v_t \defeq \Wv e_t + b_v \in\Reals^{\dout}
\, , 
\end{equation}
with $\Wk,\Wq \in\Reals^{\datt\times d_e}$, $\Wv\in\Reals^{\dout\times d_e}$. 
For simplicity's sake, we will consider that the bias vectors $b_k,b_q \in\Reals^{\datt}$ and $b_v\in\Reals^{\dout}$ are all equal to zero.

\begin{figure*}[t]
    \centering
    \includegraphics[scale=0.35]{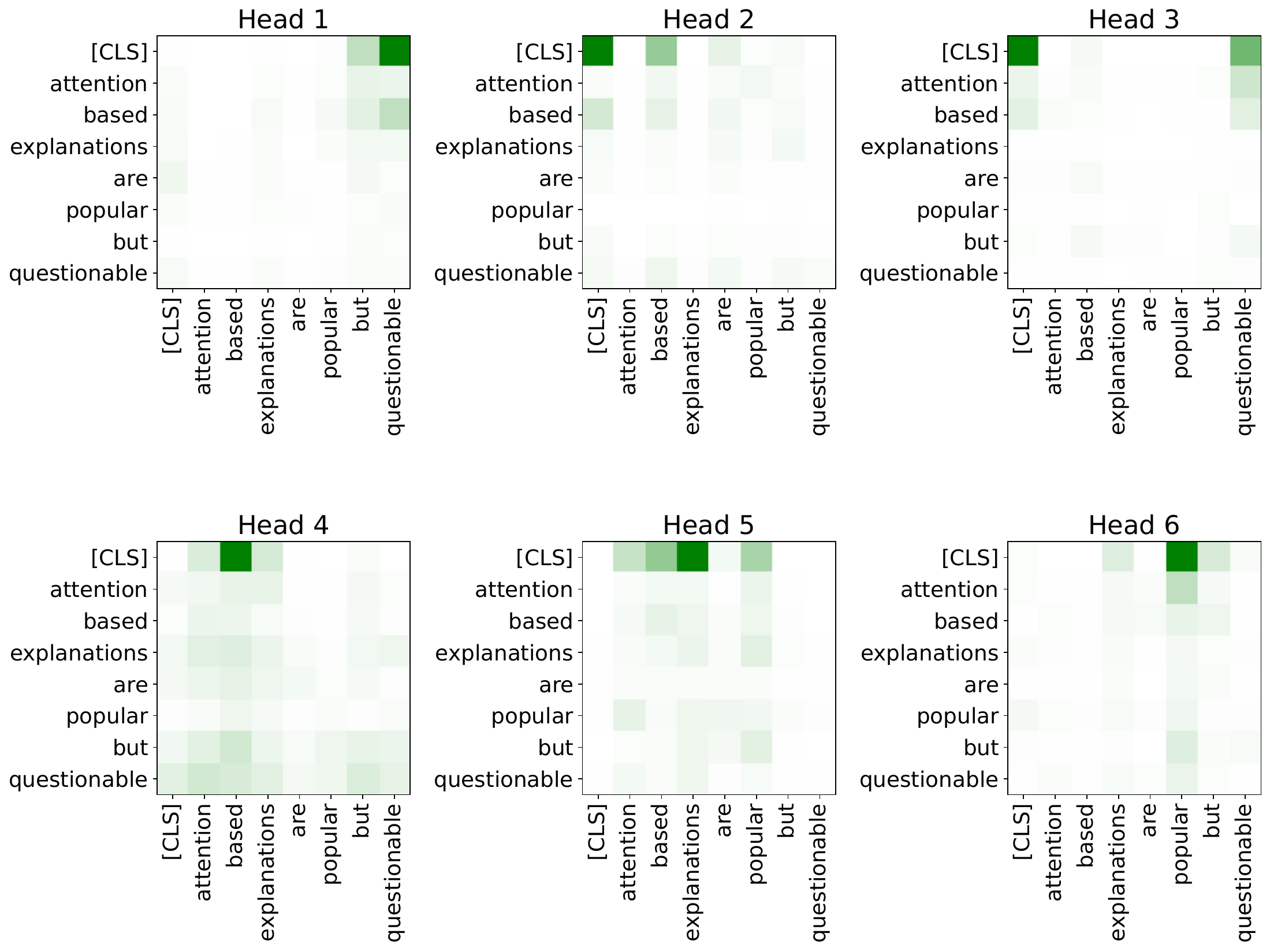}
    \caption{\label{fig:attention-heads}Attention matrices across the heads. Each head is represented by a distinct matrix, demonstrating the unique focus each head has on different parts of the document. The matrices illustrate that tokens within the document can carry significantly different weights, indicating the varying importance or relevance of each token in the context of the document. The aggregation of these weights to provide token-level scores is a critical aspect. 
    Note that Eqs. \eqref{eq:attn-avg} and \eqref{eq:attn-max} correspond to the average and the maximum values, respectively, of the first row across all six matrices.
    }
\end{figure*}

\paragraph{Attention.}
For a given query $q\in\Reals^{\datt}$, each index $t$ receives \emph{attention}
\begin{equation}
\label{eq:def-attention}
\alpha_t \defeq \frac{\exp{q^\top k_t / \sqrt{\datt}}}{\sum_{u=1}^{\tmax} \exp{q^\top k_u  /\sqrt{\datt}}}
\, .
\end{equation}
We note that, since $\Wq$ and $\Wk$ are learnable parameters of the model, there is \emph{a priori} no need for the $1/\sqrt{\datt}$ scaling factor. 
Nonetheless, it is instrumental in scaling the positional embedding properly and we keep it as is in our analysis. 

\paragraph{Output of the model.}
Finally, the intermediary output value before the final linear transformation associated to the query $q$ is 
\begin{equation}
\label{eq:def-output}
\vtilde \defeq \sum_{t=1}^{\tmax} \alpha_t v_t \in\Reals^{\dout} 
\, .
\end{equation}
Each individual head is a linear transformation of the $\vtilde$ associated to the query corresponding to the \texttt{[CLS]} token (as in \citet{devlin2019bert, huang2015bidirectional,sanh2019distilbert}). 
Namely, as defined at the beginning of this section for $i\in [K]$, $f_i(x) = \Wl^{(i)}\vtilde^{(i)}$.

We discuss limitations and the main differences between our model and practical architectures in Section \ref{sec:limitation}. 
%


\section{Attention-based explanations}
\label{sec:attn-exp}

The \textbf{scaled dot-product attention}, introduced by \citet{vaswani_et_al_2017} (corresponding to the definition of Eq.~\eqref{eq:def-attention}), essentially measures the relation among tokens. 
This results in the generation of a matrix, where each entry represents the degree of association between a pair of tokens. 
In essence, any attention-head $i \in [K]$ results in a $T\times T$ attention matrix $A^{(i)}$ (illustrated in Figure~\ref{fig:attention-heads}), where each entry $A^{(i)}_{s, t} = \alpha_t(q_s)$, \emph{i.e.}, the attention as defined in Eq.~\eqref{eq:def-attention} computed with respect to the $s$-th query token. 

Furthermore, as this study focuses on a classification model, we only consider the \texttt{[CLS]} token, which encapsulates the core of the classification \citep{chefer2021transformer}. 
Formally, this implies that only the specific query $q$ linked to the \texttt{[CLS]} token holds relevance in Eq.~\eqref{eq:def-attention}. 
This is equivalent to selecting the first row of the attention matrices in Figure~\ref{fig:attention-heads}. 

Note that, in general, a Transformer model is structured as a series of sequential layers. 
Each of these layers is equipped with a specific number of parallel heads. 
These heads operate independently, executing the attention mechanism. 
Subsequently, in order to produce token-level attention-based explanations, one must aggregate the attention matrices at both the head level and layer level. 
\citet{mylonas2023attention} offer a detailed depiction of the typical operations involved, and we specifically refer to Figure~2 in \citet{mylonas2023attention} for a comprehensive picture. 

In our scenario, the model defined in Section~\ref{sec:the-model} is single-layered, hence we omit the layer-level aggregation. 

As a result, for each head, there exists an attention vector of size $T$ that emphasizes the focus of the head on each token. 
However, as depicted in Figure~\ref{fig:attention-heads}, heads often concentrate on different sections of the document. 
Therefore, the aggregation of the $K$ attention vectors becomes a critical operation. 
The two most common operations at this level involve computing the average vector or determining the maximum value among the vectors for each token. 
Formally, we define, for any token $t \in [T]$, 
\begin{equation}
\label{eq:attn-avg}
    \attavg_t \defeq \frac{1}{K} \sum_{i=1}^K \alpha_t^{(i)} \,,
\end{equation}
and
\begin{equation}
\label{eq:attn-max}
    \attmax_t \defeq \max_{i \in [K]} \alpha_t^{(i)} \,. 
\end{equation}

Remark that, in general, 
\attavg and \attmax can lead to very different explanations. 
Additionally, it is noteworthy that \attavg and \attmax, \gradnorm{1} generate non-negative weights.  
Consequently, these methods do not differentiate between words that contribute positively or negatively to the prediction, as depicted in Figure \ref{fig:attn-xai}. 


\section{Gradient-based explanations}
\label{sec:gradient}
In this section, we delve into the realm of gradient-based explanations. 
We first recap the main methods, before computing these explanations for the model proposed in Section~\ref{sec:the-model}. 

\subsection{Methods}
\label{sec:gradient-xai}

In this section, we describe existing gradient-based methods, sometimes called saliency maps, by analogy to a similar technique in computer vision. 
Given a model $f$ and an instance $x$, by a slight abuse of language, we call the gradient with respect to a token $t \in [T]$ 
\begin{equation}
\label{eq:gradient}
\nabla_{e_t}f(x) \in \Reals^{d_e} 
\, .  
\end{equation}
It is important to note that the gradient $\nabla_{e_t}$ is calculated with respect to the embedding vector $e_t \in \Reals^{d_e}$. 
To derive per-token importance weights, several strategies exist. 
The primary approaches, falling into the class of \emph{Gradient} explanations, involve taking the mean value (\gradavg) \citep{atanasova2020diagnostic}, the $L^1$ norm (\gradnorm{1}) \citep{li2016visualizing}, or the $L^2$ norm (\gradnorm{2}) \citep{poerner2018evaluating,arras2019evaluating,atanasova2020diagnostic} of the components of Eq.~\eqref{eq:gradient}.

An alternative approach, known as \emph{Gradient times Input} (\gradtimesi) \citep{denil2014extraction}, suggests computing salience weights by performing the dot product of the gradient from Eq.~\eqref{eq:gradient} with the input word embedding $e_t$. 
In our notation, the saliency weights are thus calculated as $e_t^\top(\nabla_{e_t}f(x))$. 

While these methods share a common foundation, it is important to remark that the explanations they generate can vary significantly, and may even be contradictory. 
As illustrated in Figure~\ref{fig:attn-xai}, we observe, for instance, that \gradnorm{1} and \gradnorm{2} methods yield non-negative weights. 
In other words, these methods do not distinguish between words that contribute positively or negatively to the prediction, contrary to \gradtimesi (see Figure~\ref{fig:attn-xai}). 


\subsection{Gradient of our model}
\label{sec:gradient-f}
Let us consider the model described in Section~\ref{sec:the-model}. 
First, let us note that $f$ is linear with respect to the $f_i$ head, $i \in [K]$, hence, the gradient of $f$ with respect to the token embedding~$e_t$ is 
\begin{equation}
\label{eq:gradient-linearity}
\nabla_{e_t}f(x) \defeq \frac{1}{K} \sum_{i=1}^K \nabla_{e_t} f_i(x) \in \Reals^{d_e}
\, .
\end{equation}
The quantity of interest is thus the gradient of a single attention-head, $\nabla f_i(x)$. 
Recall that $q$ is the query corresponding to the classification token \texttt{[CLS]}. 
With this notation at hand, we can now state the following (which is proved in Appendix~\ref{sec:gradient-attention-proof}. ).

\begin{theorem}[Gradient meets attention]
\label{th:gradient-meets-attention}
The gradient of the model $f$ defined by Eq.~\eqref{eq:def-model}, with respect to the embedded token $e_t$, $t \in [T]$, is 
\begin{align}
&\nabla_{e_t} f(x) = \frac{1}{K}\sum_{i=1}^K \left[\alpha_t^{(i)} (\Wv^{(i)})^\top(\Wl^{(i)})^\top \right. \label{eq:gradient-head} \\
&\left. +\frac{\alpha_t^{(i)} }{\sqrt{\datt}} \Wl^{(i)} \left (v_t^{(i)} -\sum_{s=1}^{\tmax} \alpha_s^{(i)} v_s^{(i)} \right )(\Wk^{(i)})^\top q\right] \in\Reals^{d_e} \notag
\, . 
\end{align}
\end{theorem}

By substituting Eq.~\eqref{eq:gradient-head} into Eq.~\eqref{eq:gradient-linearity}, we are able to reconstruct the gradient-based explanations discussed in Section~\ref{sec:gradient-xai}. 
Indeed, \gradavg, \gradnorm{1}, and \gradnorm{2}, are nothing but the average, the $L^1$, and the $L^2$ norm of $\nabla_{e_t} f(x) \in \Reals^{d_e}$, respectively, while \gradtimesi is the dot product between the gradient and the embedding vector: $e_t^\top(\nabla_{e_t}f(x))$. 

We now make a few comments. 
(i) the gradient of $f$, at the first order approximation, is \textbf{linear in $\alpha$}, which can explain the correlation with the attention weights to some extent. 
Consequently, \textbf{\attavg is correlated with \gradavg and \gradnorm{1}}, as by desiderata (i) of \citet{jain2019attention}, while the same does not hold for \attmax and \gradnorm{2}. 
We want to emphasize that this may not necessarily hold true for deeper models.
(ii) \textbf{the gradient captures the influence of the linear layers $\Wl^{(i)}$}: we believe that this is a useful insight, disregarded by  attention-based explanations. 
For instance, let us assume $K=1$ and $v_t \approx \sum_{s=1}^{\tmax} \alpha_s v_s$ (corresponding to a situation where the value vector is ``typical''). 
Then the only remaining part in Eq.~\eqref{eq:gradient-head} is $\alpha_t \Wv^\top\Wl^\top$. 
A positive $\alpha_t$ can give rise to \textbf{negative} explanations if the coordinates of $\Wv^\top\Wl^\top$, reflecting the true behavior of the model.


\section{Perturbation-based explanations: the example of LIME}
\label{sec:lime}

Let us now turn towards perturbation-based explanations. 
We focus on LIME for text data, first recalling how it operates, introducing additional notation on the way, then stating our main result. 


\subsection{Reminder on LIME}

Here we give a short introduction to LIME for text data in our context, following closely \citet{mardaoui_garreau_2021}. 
The overall idea underlying LIME for text data is to start from $\xi$ the document to explain and produce local perturbations $X_1,\ldots,X_n$. 
From these perturbations, a local linear model trying to fit the predictions of $f$ is trained, and the linear coefficients corresponding to this linear model given as explanation to the user. 

\paragraph{Sampling.}
Let us call $X$ the distribution of the randomly perturbed documents.
Then, in our notation, $X$ is generated as follows: first pick $s$ uniformly at random in $[d]$ (the local dictionary), then chose a set $S\subseteq [d]$ of size $s$ uniformly at random. 
Finally, remove from $\xi$ all occurrences of words appearing in $S$.  
Here, removing means replacing by the \texttt{UNK} token. 
In the present work, for simplicity, we assume that tokens and words coincide. 
The perturbed samples $X_1,\ldots,X_n$ are i.i.d. repetitions of this process. 
Associated to the $X_i$s we have vectors $Z_1,\ldots,Z_n\in\{0,1\}^d$, marking the absence or presence of a word in $X_i$. 
Namely, we set $Z_{i,j}=1$ if word $j$ belongs to $X_i$ and $0$ otherwise. 

\paragraph{Weights.}
Each new sample~$X_i$ receives a positive weight~$\pi_i$, defined by
\begin{equation}
\label{eq:def-weights}
\pi_i \defeq \exp{\frac{-d(\Indic,Z_i)^2}{2\nu^2}}
\, ,
\end{equation}
where $d$ is the \emph{cosine distance} and $\nu >0$ is a bandwidth parameter. 
The intuition behind these weights is that $X_i$ can be far away from $\xi$ if many words are removed (in the most extreme case, $s=d$, all the words from $\xi$ are removed). 
In that case, $z_i$ has mostly $0$ components, and is far away from $\Indic$ from the point of view of the cosine distance.

\paragraph{Surrogate model.}
The next step is to train a surrogate model on $Z_1,\ldots,Z_n$, trying to match the responses $Y_i\defeq f(X_i)$. 
In the default implementation of LIME, this model is linear and is obtained by weighted ridge regression. 
Formally, LIME outputs 
\begin{equation}
\label{eq:main-problem}
\betahat_n^{\lambda} \in\Argmin_{\beta\in\Reals^{d+1}} \biggl\{ \sum_{i=1}^n \pi_i(y_i - \beta^\top z_i)^2 + \lambda \norm{\beta}^2\biggr\}
\, ,
\end{equation}
where $\lambda>0$ is a regularization parameter. 
We call the components of $\betahat_n^\lambda$ the \emph{interpretable coefficients}, the $0$th coordinate in our notation is by convention the intercept. 


\subsection{Limit explanations for our model}

\begin{figure}
\centering
\includegraphics[scale=0.35]{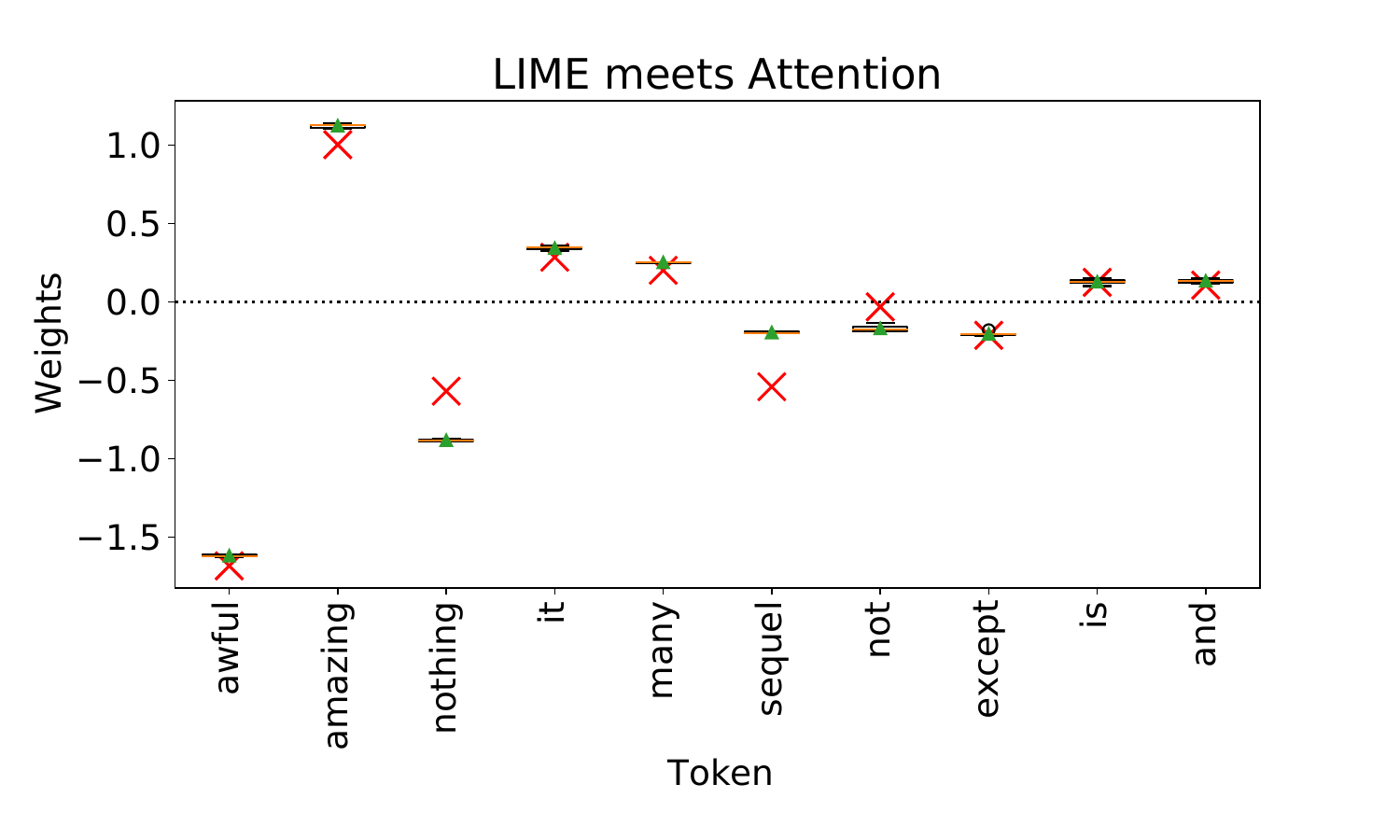}
\vspace{-0.2in}
\caption{\label{fig:accuracy-lime-meets-attention}Illustration of the accuracy of Eq.~\eqref{eq:lime-meets-attention}. 
The boxplots show the results from $5$ runs of LIME with default parameters, while the red crosses indicate the predictions given by Theorem~\ref{th:lime-meets-attention}. 
The document $\xi$ contains $T=99$ tokens and $d=71$ distinct words and is classified as a negative review. 
Note that Theorem~\ref{th:lime-meets-attention} holds true even with $T\neq d$ for reasonable word multiplicities, as discussed in Section \ref{sec:limitation}. 
}
\end{figure}

Under mild assumptions, \citet[Theorem~1,][]{mardaoui_garreau_2021} show that LIME's coefficients converge to \emph{limit coefficients} $\beta^\infty$.  
Namely, the number of perturbed samples $n$ is large, the penalization in Eq.~\eqref{eq:main-problem} is not too strong (which is the case by default), and the bandwidth $\nu$ is also large. 
The expression for the limit coefficient associated to word~$j$, 
\begin{equation}
\label{eq:limit-lime-coef}
\betainf_j \!=\! 3\condexpec{f(X)}{j\notin S} - \frac{3}{d}\sum_k \condexpec{f(X)}{k\notin S} 
\, ,
\end{equation}
can then be computed (exactly or approximately) as a function of the parameters of the model to gain some precise insights on the behavior of LIME in this situation. 
This computation is the main result of this section. 

Before stating Theorem~\ref{th:lime-meets-attention}, we need some additional notation. 
Indexing by $h$ will denote the quantity corresponding to the \texttt{[UNK]} token. 
In particular, $k_{h,t}\defeq \Wk h + \Wk\Wp(t) \in \Reals^{\datt}$ is the key vector associated to the \texttt{[UNK]} token at position $t\in [\tmax]$. 
For any $t\in [\tmax]$, we further define
\begin{equation}
\label{eq:def-g-values}
g_{h,t} \defeq \exp{q^\top k_{h,t} / \sqrt{\datt}}
\, ,
\end{equation}
and
\[
\alpha_{h,t} \defeq \frac{g_{h,t}}{\sum_u g_{h,u}}
\, . 
\]
We note that $\alpha_{h,t}$ can be seen as the attention corresponding to the \texttt{[UNK]} token at position $t$ from the perspective of the query associated to the \texttt{[CLS]} token. 
Finally, set $v_{h,t} \defeq \Wv (h+\Wp(t))$. 
We have:

\begin{theorem}[LIME meets attention]
\label{th:lime-meets-attention}
Assume that $d=T=\tmax^{\epsilon}$, with $\epsilon \in (0,1)$. 
Assume further that there exist positive constants $0<c<C$ such that, as $T\to +\infty$, for all $t\in [\tmax]$, $\max(\abs{v_t},\abs{v_{h,t}})\leq C$, and $c \leq \min(g_t,g_{h,t})\leq C$. 
Then 
\begin{align}
\label{eq:lime-meets-attention}
\betainf_j &= \frac{3}{2K}\sum_{i=1}^K\sum_{t=1}^{\tmax}  \Wl^{(i)} \left(\alpha_t^{(i)}v_t^{(i)} - \alpha_{h,t}^{(i)} v_{h,t}^{(i)}\right)\indic{\xi_t = j} \notag \\
&+ \bigo{\tmax^{(2-\epsilon)\vee 3/2}} 
\, .
\end{align}
\end{theorem}

Theorem~\ref{th:lime-meets-attention} is proved in Section~\ref{sec:lime-meets-attention-sketch-of-proof}. 
The challenging part of the proof is to derive good approximations for Eq.~\eqref{eq:limit-lime-coef}, since the model we consider, although single-layered, is highly non-linear. 
Note that Theorem \ref{th:lime-meets-attention} relies on the assumption that all tokens are distinct. 
While we conjecture that this assumption can be relaxed (as discussed in Section~\ref{sec:limitation}), it is necessary for a rigorous comparison between attention and LIME explanations. 
However, our experiments were conducted without assuming all tokens are distinct.
We nevertheless observe a very good match between our theoretical approximation and the empirical outputs of LIME with default parameters, which we illustrate in Figure~\ref{fig:accuracy-lime-meets-attention} on a particular example, and in a more quantitative way in Appendix~\ref{sec:lime-meets-attention-sketch-of-proof}. 

We now make a few comments. 
(i) it is clear from Eq.~\eqref{eq:lime-meets-attention} that LIME explanations are \textbf{quite different} from gradient-based explanations (Eq.~\eqref{eq:gradient-head}), with the exception of the leading term (which is proportional to $\alpha_t\Wl v_t$) which we recognize in both expressions.
(ii) one can see that \textbf{LIME explanations are approximately an affine transformation of the $\alpha_t^{(i)}$}. 
(iii) as it is the case for gradient-based explanation, there is a major difference with plain attention-based explanations: the last layer comes into account in the explanation. 
To put it plainly, let us assume that $K=1$ and that $\Wl v_t =0$, \textbf{the influence of $\alpha_t$ disappears whatever its value}. 
We see this as an advantage for LIME, since this situation corresponds to the model killing the influence of token $t$ in later stages, although a positive attention score is given. 
(iv) from Eq.~\eqref{eq:lime-meets-attention}, we see that \textbf{LIME explanations will be near zero} whenever $\alpha_t^{(i)}v_t^{(i)} \approx \alpha_{h,t}^{(i)} v_{h,t}^{(i)}$, a scenario in which the \textbf{attention $\times$ value of head $i$ given to token $t$ if comparable to that of the attention given to the \texttt{[UNK]} token}. 
This makes a lot of sense, while calling for a careful choice of embedding for the replacement token. 



\section{Limitation}
\label{sec:limitation}
The main differences between the model described in Section \ref{sec:the-model} and practical architectures are the following: (i) number of layers, (ii) skip connections, (iii) non-linearities. 
(i): we only consider a single layer, which already brings non-linearity while giving a quite challenging analysis and a very good performance in practice for our task. 
(ii): we do not consider skip connections since we did not observe an increase in performance while adding them, but our analysis can easily be adapted to this setting since this operation is linear. 
%
(iii): in line with typical theoretical works \citep{gunasekar2018implicit}, we refrain from incorporating additional non-linearities; specifically, we do not introduce a ReLU activation or any of its variants.

The primary limitation of the present analysis lies in its focus on a single-layer model. 
This facilitates a deep theoretical exploration, yet its applicability to more intricate architectures may not be direct. 
The extension of this analysis to a multi-layer architecture introduces additional theoretical challenges. 
In the context of multi-layer transformers, approximation errors have the potential to accumulate and intensify as they traverse through the network. 
Moreover, the assumptions applicable to a single layer may not necessarily hold true for deeper networks, especially those incorporating non-linearities such as ReLU activations. 
However, even for simple architectures, many questions remain unresolved, and the interpretability of these models has not been thoroughly studied in a formal manner. 

Several theoretical studies on transformers share these limitations. 
For example, \citet{jelassi2022vision} elucidates how Vision Transformers discern spatial patterns within a single-layer, single-head architecture. 
Similarly, \citet{tarzanagh2023transformers} establishes the correspondence between the optimization geometry of self-attention and an SVM problem. 
\citet{von2023transformers} suggests that training Transformers with a specific objective can induce a form of meta-learning, exemplified on a linear single-layer model. 
Additionally, \citet{edelman_et_al_2022} investigate transformers in time series forecasting, showcasing their superiority over traditional methods by accurately capturing temporal dependencies. 
In a bid to enhance transformer efficiency and scalability, \citet{fu_et_al_2023} introduce sparse attention mechanisms and efficient training strategies by formalizing single-layer transformers. 
Furthermore, \citet{makkuva_et_al_2024} tackle transformer interpretability, developing tools to visualize and comprehend attention patterns within the models, aiming to bridge the gap between high performance and the imperative for transparency in critical applications. 
In Appendix \ref{sec:appendix-multi-layer}, we report experiments on a multi-layer transformer. 

We conclude this section by discussing limitations of Theorem~\ref{th:lime-meets-attention}. 
First, as in previous work, the approximation is only true for large document and window size. 
This comes without surprise, but we note that the results are experimentally satisfying for documents which are a few dozen tokens long. 
Second, we assume in Theorem~\ref{th:lime-meets-attention} that all tokens are distinct. 
This assumption, though a simplification, allows for a rigorous formalization of LIME's behavior using its default parameters as defined in the official implementation. 
Experimentally, Eq.~\eqref{eq:lime-meets-attention} holds for documents containing repeated words. 
We validated Theorem \ref{th:lime-meets-attention} disregarding this assumption, by computing the norm-$2$ error between the official LIME weights and our approximation (as illustrated in Figure \ref{fig:accuracy-lime-meets-attention}): the average error over the full test set (described in Appendix \ref{sec:experiment}) is $0.808$, with a standard deviation of $0.219$. 
From a theoretical aspect, we conjecture that this assumption can be relaxed if the maximal multiplicity of tokens is small relative to $T$. 
If many tokens are identical, this is no longer true: consider, for instance, the extreme case of two groups of identical tokens. 
We delve into this topic in more detail in Appendix \ref{sec:lime-meets-attention-discussion}.


\section{Conclusion and future work}
\label{sec:conclusion}
In this paper, we offered a theoretical analysis on how
post-hoc explanations relates to a single-layer multi-head
attention-based network. 
Our work contributes to the ongoing debate in this area by providing exact and approximate expressions for post-hoc explanations on such model. 
Through these expressions, we were able to highlight the fundamental differences between attention-based, gradient-based, and perturbation-based explanations. 
This deeper understanding not only enriches the ongoing discourse surrounding interpretability but also offers valuable insights for practitioners and researchers navigating the complexities of transformers' interpretation. 

It is crucial to acknowledge that the quest for perfect explanations remains elusive; no single method has emerged as entirely satisfactory. 
However, it is clear that current models employ attention scores in a non-intuitive manner to arrive at the final prediction. 
In particular, these scores go through a series of further transformations, which is ignored when looking solely at attention scores. 
These scores also always provide a positive explanation, in contrast to (most) perturbation-based and gradient-based approaches.
For these reasons, we believe that they can extract more valuable insights than a mere examination of attention weights. 
This finding aligns with the assertions made by \citet{bastings2020elephant}. 

As future work, we plan to broaden the scope of our analysis by extending our investigations to diverse range of post-hoc interpretability methods, including Anchors, thus understanding model explanations across different methodologies. 
We also would like to obtain similar statements (connecting explanations to the parameters of the model) for more complicated architectures, including skip connections, additional non-linearities, and multi-layer models, enabling us to discern the relationship between model parameters and different explanations. 
Addionally, there is some interplay between the sampling mechanism of perturbation-based methods (often replacing at the word level) and the tokenizer used by the model (tokens are often subwords) which we would like to understand better. 
Lastly, we emphasize that our focus in this paper has been on text classification. 
This choice allows us to capitalize on well-established, and broadly studied post-hoc explainers and conduct a thorough theoretical analysis based on this specific domain. 
However, we intend to expand the scope of applications for our analysis. 
Specifically, we remark that our study focused on token-level explanations. 
Moving forward, we intend to extend our findings beyond text models to encompass other domains, such as computer vision. 


\section*{Acknowledgements}
This work has been supported by the French government, through the NIM-ML project (ANR-21-CE23-0005-
01), and by EU Horizon 2020 project AI4Media (contract
no. 951911). 
Most of this work was realized while DG was employed at Universit\'e C\^ote d'Azur. 

\section*{Impact Statement}
This paper presents work whose goal is to advance the field of Machine Learning. 
There are many potential societal consequences of our work, none which we feel must be specifically highlighted here. 


\bibliography{paper}
\bibliographystyle{icml2024}

\newpage
\appendix
\onecolumn

\icmltitle{Appendix for the paper \\ {Attention Meets Post-hoc Interpretability: A Mathematical Perspective}}

\paragraph{Organization of the Appendix.}
We start by providing proofs for the theoretical results presented in the paper. 
We prove Theorems \ref{th:gradient-meets-attention} and \ref{th:lime-meets-attention} in Sections~\ref{sec:gradient-attention-proof} and \ref{sec:lime-meets-attention-sketch-of-proof}, respectively. 
Sections \ref{sec:appendix:proof-key-prop-lime-approx} and \ref{sec:appendix:technical} collect additional technical details crucial for the proofs. 
Finally, in Section~\ref{sec:experiment}, we detail the model employed for our experiments. 
For further information, the training and experimental code are available at \url{https://github.com/gianluigilopardo/attention_meets_xai}.


\section{Proof of Theorem~\ref{th:gradient-meets-attention}}
\label{sec:gradient-attention-proof}

In this section, we show how to compute the gradient of $f$ with respect to the embedding $e_t$, $t \in [T]$. 
By linearity, we can focus on one head $f_i$, $i \in [K]$, and thus we momentarily drop the $i$ superscripts.  
Let us start by computing the gradients of the key, query, and value vectors $k_t$, $q_t$, and $v_t$ (Eqs.~\eqref{eq:def-key}, \eqref{eq:def-query}, and \eqref{eq:def-value}).
For any $t \in [T]$, one has
\begin{align}
\label{eq:gradient-key}
    \nabla_{e_t} k_t = \nabla_{e_t} \left(\Wk e_t\right) = \Wk \left(\nabla_{e_t} e_t\right) =
        \Wk^\top \in \Reals^{d_e \times \datt} 
    \, ,     
\end{align}
\begin{align}
\label{eq:gradient-query}
    \nabla_{e_t} q_t = \nabla_{e_t} \left(\Wq e_t\right) = \Wq \left(\nabla_{e_t} e_t\right) =
        \Wq^\top \in \Reals^{d_e \times \datt} 
    \,,     
\end{align}
and 
\begin{align}
\label{eq:gradient-value}
    \nabla_{e_t} v_t = \nabla_{e_t} \left(\Wv e_t\right) 
    =  \Wv^\top \in \Reals^{d_e \times \dout} 
    \,.      
\end{align}

Therefore, the gradient of the attention $\alpha_t$ as defined in Eq.~\eqref{eq:def-attention}, for any $t \in [T]$, 
\begin{align*}
    \nabla_{e_t} \alpha_t & = \nabla_{e_t} \left(\frac{\exp{q^\top k_t / \sqrt{\datt}}}{\sum_{u=1}^{T_{max}} \exp{q^\top k_u / \sqrt{\datt}}} \right) \\
    & = \frac{\alpha_t}{\sqrt{\datt}} \left((\Wk^\top q) - \sum_{s=1}^{\tmax} \alpha_s (\Wk^\top q) \right) \quad \in \Reals^{d_e} 
    \, .
\end{align*}
The situation is similar if we look at another attention coefficient: let $s\neq t$, then
\begin{align*}
\nabla_{e_t}\alpha_s &= \nabla_{e_t} \left( \frac{\exp{q^\top k_s / \sqrt{\datt}}}{\sum_{u=1}^{T_{max}} \exp{q^\top k_u / \sqrt{\datt}}}\right) \\
&= \frac{-\alpha_u\alpha_t}{\sqrt{\datt}} (\Wk^\top q)
\, .
\end{align*}

Finally, we can compute the gradient of $\vtilde$ as
\begin{align*}
\label{eq:gradient-vtilde}
\nabla_{e_t} \vtilde & = \nabla_{e_t} \left(\sum_{u=1}^{\tmax} \alpha_u v_u\right) \\
&    = \nabla_{e_t}(\alpha_t) v_t + \alpha_t (\nabla_{e_t}v_t) + \sum_{s\neq t} (\nabla_{e_t}\alpha_s) v_s\\
&= \frac{1}{\sqrt{\datt}}(\Wk^\top q)(\alpha_t - \alpha_t^2) v_t + \alpha_t \Wv^\top + \sum_{s\neq t} \frac{-\alpha_u\alpha_t}{\sqrt{\datt}} (\Wk^\top q) \\
\nabla_{e_t} & = \frac{\alpha_t}{\sqrt{\datt}} \left(v_t - \sum_{s=1}^{\tmax} \alpha_s v_s \right)(\Wk^\top q)  + \alpha_t \Wv^\top 
    \,.
\end{align*}
Finally, since $f(x) = \Wl \vtilde$, we deduce Eq.~\eqref{eq:gradient-head} from the last display, multiplying by $\Wl^\top$ and averaging. 
\qed 

Theorem~\ref{th:gradient-meets-attention} is also true in practice, as illustrated in Figure \ref{fig:gradient-meets-attention}.

\begin{figure}
    \centering
    \includegraphics[scale=0.33]{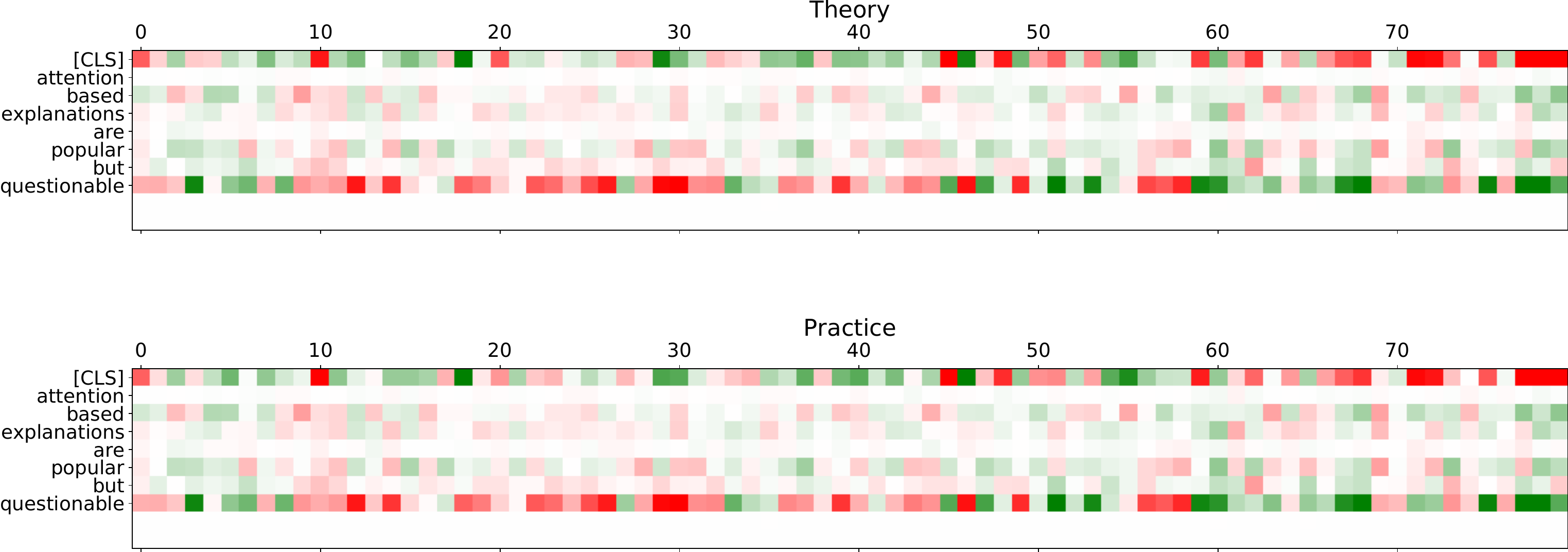}
    \caption{\label{fig:gradient-meets-attention}Illustration of the accuracy of Theorem~\ref{th:gradient-meets-attention}. 
    Here, for illustrative purpose, $d_e=80$.}
\end{figure}


\section{Proof of Theorem~\ref{th:lime-meets-attention}}
\label{sec:lime-meets-attention-sketch-of-proof}

\paragraph{Preliminaries.}
The key idea of this proof is to leverage Eq.~\eqref{eq:limit-lime-coef} and find a good approximation for the conditional expectations involved. 
Looking closer at Eq.~\eqref{eq:limit-lime-coef}, we first notice that, by linearity, we can focus on the limit coefficients associated to a single head. 
Thus we drop the $i$ indexation in this proof. 

Now let us recall that $S$ is the random subset of words from the dictionary being removed when generating $X$. 
Our first key observation is that \textbf{$X$ has random token embeddings $E_t$}. 
More precisely, Eq.~\eqref{eq:def-embedding} becomes
\begin{equation}
\label{eq:random-embedding}
\forall t\in [\tmax],\quad E_t \defeq e_t\indic{\xi_t\notin S} + (h+\Wp(t))\indic{\xi_t\in S}
\, .
\end{equation}
We note that $E_t$ for $t>T$ is actually not random (LIME does not perturb outside of $\xi$), but this will be of no consequence. 
In turn, keys and queries are modified, that is, Eqs.~\eqref{eq:def-key} and~\eqref{eq:def-query} become, respectively,
\begin{equation}
\label{eq:def-random-keys}
\forall t\in [\tmax],\quad K_t \defeq k_t\indic{\xi_t\notin S} + \Wk (h+\Wp(t))\indic{\xi_t\in S}
\, ,
\end{equation}
and
\begin{equation}
\label{eq:def-random-querys}
\forall t\in [\tmax],\quad Q_t \defeq q_t\indic{\xi_t\notin S} + \Wq (h+\Wp(t))\indic{\xi_t\in S}
\, .
\end{equation}
The attention coefficients associated to the \texttt{[CLS]} token also become random.
In analogous fashion to Eq.~\eqref{eq:def-g-values}, let us define
\[
\forall u\in [\tmax], \quad G_u \defeq \exp{q^\top K_u / \sqrt{\datt}}
\, ,
\]
where we recall that $q\in\Reals^{\datt}$ is the query vector associated to the \texttt{[CLS]} token. 
Taking dot product and exponential, and noting that the indicator functions concern disjoint events, we see that 
\begin{equation}
\label{eq:random-g}
\forall u\in [\tmax], \quad G_u = g_u \indic{\xi_u \notin S} + g_{h,u} \indic{\xi_u \in S}
\, ,
\end{equation}
where we let $g_u\defeq \exp{q^\top k_u/ \sqrt{\datt}}$ and $g_{h,t} =  \exp{q^\top k_{h,t}  / \sqrt{\datt}}$ as in Eq.~\eqref{eq:def-g-values}. 
Then, with this notation in hand, we define the random attention coefficient associated to token $t$ by
\begin{equation}
\label{eq:def-random-attention}
\forall t \in [\tmax], \quad A_t \defeq \frac{G_t}{\sum_{u=1}^{\tmax} G_u}
\, .
\end{equation}
Finally, value vectors are also random in this setting. 
Namely,
\begin{equation}
\label{eq:def-random-v}
\forall t\in [\tmax],\quad V_t \defeq v_t\indic{\xi_t\notin S} + v_{h,t}\indic{\xi_t \in S}
\, ,
\end{equation}
where we recall that $v_{h,t} = \Wv (h+\Wp(t))$.

\paragraph{Reduction to key computation.}
Looking at Eq.~\eqref{eq:limit-lime-coef}, and now with appropriate notation, we need to compute 
\[
\condexpec{f(X)}{\ell \notin S} = \condexpec{\sum_{t=1}^{\tmax} A_{t} V_t}{\ell \notin S}
\]
for all $\ell\in [d]$. 
Again by linearity, one can focus on the computation of $\condexpec{A_{t} V_t}{\ell \notin S}$. 
The following result gives an approximation of this quantity when both $\tmax$ and $d$ are large:


\begin{proposition}[Approximated conditional expectation]
\label{prop:approx-cond-expec}
Assume that $d=T$. 
Assume further that there exist positive constants $0<c<C$ such that, as $T\to +\infty$, for all $t\in [\tmax]$, $\max(\abs{v_t},\abs{v_{h,t}})\leq C$, and $c \leq \min(g_t,g_{h,t})\leq C$. 
Then, for any $t\in [\tmax]$, if $\xi_t=\ell$, 
\begin{equation}
\label{eq:cond-expec-aux-1}
\condexpec{A_tV_t}{\ell\notin S} = 
\frac{1}{d}\sum_{s=1}^{d-1}
\frac{g_tv_t}{ \left(1-\frac{s}{d-1}\right)\sum_u g_u + \frac{s}{d-1}\sum_u g_{h,u}}
+ \bigo{\tmax^{-3/2}}
\, ,
\end{equation}
and otherwise
\begin{equation}
\label{eq:cond-expec-aux-2}
\condexpec{A_tV_t}{\ell\notin S} = 
\frac{1}{d}\sum_{s=1}^{d-1}
\frac{\left(1-\frac{s}{d-1}\right)g_tv_t + \frac{s}{d-1}g_{h,t}v_{h,t}}{ \left(1-\frac{s}{d-1}\right)\sum_u g_u + \frac{s}{d-1}\sum_u g_{h,u}}
+ \bigo{\tmax^{-3/2}}
\, .
\end{equation}
\end{proposition}

The proof of Proposition~\ref{prop:approx-cond-expec} is deferred to Section~\ref{sec:appendix:proof-key-prop-lime-approx}. 
From Eqs.~\eqref{eq:cond-expec-aux-1} and~\eqref{eq:cond-expec-aux-2}, coming back to Eq.~\eqref{eq:limit-lime-coef}, we deduce that the $j$th limit coefficient associated to $A_tV_t$ is approximately equal to
\begin{equation}
\label{eq:key-computation-limit-coef}
\frac{3}{d} \sum_{s=1}^{d-1} \frac{\frac{s}{d-1}(g_tv_t - g_{h,t}v_{h,t})}{\left(1-\frac{s}{d-1}\right)\sum_u g_u + \frac{s}{d-1}\sum_u g_{h,u}} + \bigo{\tmax^{-3/2}}
\, .
\end{equation}
The derivative of the mapping 
\[
x\mapsto \frac{x}{(1-x) \sum_u g_u + x \sum_u g_{h,u}}
\]
is given by 
\[
x\mapsto \frac{\sum_u g_u }{(x(\sum_u g_{h,u} - \sum_u g_u ) + \sum_u g_u)^2}
\, .
\]
On $[0,1]$, under our assumptions, the last display is uniformly bounded by $\bigo{\tmax^{-1}}$ in absolute value. 
Thus, by standard Riemann sum approximation, the last display is
\[
3(g_tv_t - g_{h,t}v_{h,t})\int_0^1 \frac{x\Diff x}{(1-x)\sum_u g_u + x\sum_u g_{h,u}} + \bigo{\tmax^{-3/2}}
\, .
\]
Let us recall that, if $T<u\leq \tmax$, $g_u=g_{h,u}$. 
Therefore, $\sum_u g_{h,u} - \sum_u g_u = \bigo{T} = \bigo{\tmax^\epsilon}$, and $(\sum_u g_{h,u} - \sum_u g_u)/\sum_u g_u = \bigo{\tmax^{\epsilon - 1}}$. 
Therefore, according to Lemma~\ref{lemma:integral-approximation}, the integral in the last display can be well approximated by 
\[
\frac{1}{\sum_u g_u} \cdot \left( \frac{1}{2} + \bigo{\frac{\sum_u g_{h,u} - \sum g_u}{\sum g_u}}\right) =\frac{1}{2\sum g_u} + \bigo{\tmax^{\epsilon - 2}}
\, .
\]
The same reasoning shows that, whenever $\xi_t\neq j$, the approximation is zero (with the same precision in the error).  
Thus, by linearity (over the tokens and the model), we obtain the statement of Theorem~\ref{th:lime-meets-attention}. 
\qed 

\subsection{Discussion on Theorem~\ref{th:lime-meets-attention}}
\label{sec:lime-meets-attention-discussion}
A theoretical limitation of Theorem~\ref{th:lime-meets-attention} it the assumption of distinct tokens.
This assumption, while technically a simplification, enables a rigorous formalization of LIME's behavior using its default parameters as defined in the official implementation.
We conducted quantitative experiments that disregarded this assumption, and the results still hold. 
We empirically validate the accuracy of Theorem \ref{th:lime-meets-attention} by computing the norm-$2$ error between the LIME weights from the official implementation (available on Github at \url{https://github.com/marcotcr/lime}) and our approximation. 
The average norm-$2$ error, computed over the full test set (see Section~\ref{sec:experiment}), is $0.808$, with a standard deviation of $0.219$. 

The primary challenge in proving a formal result that allows for repetitions lies in the use of Lemma \ref{lemma:expected-frac}. 
The key intuition in the current proof mechanism is that if only one element of the denominator of $A_t$ varies randomly, this has a minimal overall effect on the entire denominator, given that it has $T \gg 1$ terms. 
However, if many tokens are identical, this is no longer true, and it can result in high variance (consider, for instance, the extreme case of two groups of identical tokens), which prevents us from using Lemma \ref{lemma:expected-frac}. 
Nevertheless, we conjecture that if the tokens are not distinct, but the maximal multiplicity of tokens is small relative to $T$, our findings hold true (and this is empirically true).

On a more practical note, we highlight that by default, LIME-text perturbs input data by removing all occurrences of individual words or characters (\texttt{bow=True}, \emph{i.e.}, bag-of-words), and this is the subject of our study. 
However, if the underlying model uses word location (as in our classifier), a possibility is to set \texttt{bow=False} (as recommended in their notebook), so that any occurrence of the same word is considered a distinct token. 
By using this option, the same applies in Theorem \ref{th:lime-meets-attention}: the same words in different parts of the text are considered different tokens.


\section{Proof of Proposition~\ref{prop:approx-cond-expec}}
\label{sec:appendix:proof-key-prop-lime-approx}

\paragraph{Sketch of the proof.}
Essentially, assuming for a second that $V_t$ is constant, the crux of the result is to compute the (approximate) expectation of $A_t$, which is defined as the ratio of positive quantities (Eq.~\eqref{eq:def-random-attention}).
Since the denominator is quite large, one can use Lemma~\ref{lemma:expected-frac} and approximate the expected ratio by the ratio of expectation. 
This works only if, concurrently, the variance is not too high, which is guaranteed by Lemma~\ref{lemma:conditional-variance-fixed}. 
Lemmas~\ref{lemma:expected-frac} and~\ref{lemma:conditional-variance-fixed} are stated and proved in Section~\ref{sec:appendix:technical}. 


\paragraph{Proof of Proposition~\ref{prop:approx-cond-expec}.}
We first write
\begin{align*}
\condexpec{A_tV_t}{\ell\notin S} &= \condexpec{\frac{G_tV_t}{\sum_{u=1}^{\tmax} G_u}}{\ell\notin S} \tag{Eqs.~\eqref{eq:def-random-attention} and~\eqref{eq:def-random-v}}\\
&= \condexpec{\frac{g_tv_t\indic{\xi_t\notin S} + g_{h,t}v_{h,t}\indic{\xi_t\in S}}{\sum_{u=1}^{\tmax} \left\{ g_u\indic{\xi_u\notin S} + g_{h,u}\indic{\xi_u\in S}\right\}}}{\ell\notin S} \tag{Eq.~\eqref{eq:random-g}} \\
&= \frac{1}{d} \sum_{s=0}^{d-1} \condexpecunder{\frac{g_tv_t\indic{\xi_t\notin S} + g_{h,t}v_{h,t}\indic{\xi_t\in S}}{\sum_{u=1}^{\tmax} \left\{ g_u\indic{\xi_u\notin S} + g_{h,u}\indic{\xi_u\in S}\right\}}}{\ell\notin S}{s} \tag{law of total expectation}
\, .
\end{align*}
Note that there is no $s=d$ term in the last display, since $d$ removals is incompatible with $\ell\notin S$. 
Let us set $s\in [d-1]$. 
Define $X\defeq g_tv_t\indic{\xi_t\notin S} + g_{h,t}v_{h,t}\indic{\xi_t\in S}$ and $Y\defeq \sum_{u=1}^{\tmax} \left\{ g_u\indic{\xi_u\notin S} + g_{h,u}\indic{\xi_u\in S}\right\}$. 
Under our assumptions, $X$ is clearly bounded while $Y$ has order $\tmax$. 
Thus the hypotheses of Lemma~\ref{lemma:expected-frac} are satisfied with $n=\tmax$.  
Moreover, Lemma~\ref{lemma:conditional-variance-fixed} guarantees that $\condvarunder{Y}{s}{\ell \notin S}=\bigo{\tmax}$. 
From Lemma~\ref{lemma:expected-frac}, we deduce that 
\begin{equation}
\label{eq:control-single-term}
\condexpecunder{\frac{g_tv_t\indic{\xi_t\notin S} + g_{h,t}v_{h,t}\indic{\xi_t\in S}}{\sum_{u=1}^{\tmax} \left\{ g_u\indic{\xi_u\notin S} + g_{h,u}\indic{\xi_u\in S}\right\}}}{\ell\notin S}{s} = 
\frac{\condexpecunder{g_tv_t\indic{\xi_t\notin S}+ g_{h,t}v_{h,t}\indic{\xi_t\in S}}{\ell\notin S}{s}}{\condexpecunder{\sum_{u=1}^{\tmax} \left\{ g_u\indic{\xi_u\notin S} + g_{h,u}\indic{\xi_u\in S}\right\}}{\ell \notin S}{s} } + \bigo{\tmax^{-3/2}}
\, .
\end{equation}
Let us assume from now on that $\xi_t=\ell$ (the case $\xi_t\neq \ell$ is similar). 
Then 
\begin{equation}
\label{eq:expec-numerator}
\condexpecunder{g_tv_t\indic{\xi_t\notin S}+ g_{h,t}v_{h,t}\indic{\xi_t\in S}}{\ell\notin S}{s} = g_tv_t
\, ,
\end{equation}
and for all $u\neq t$, 
\[
\condexpecunder{g_u\indic{\xi_u\notin S}+ g_{h,u}\indic{\xi_u\in S}}{\ell\notin S}{s} = g_u\condprobaunder{\xi_u\notin S}{s}{\ell\notin S} + g_{h,u}\condprobaunder{\xi_u\in S}{s}{\ell\notin S}
\, .
\]
Since we assumed distinct tokens ($d=T$), Lemma~\ref{lemma:exact-proba-computations-conditional} yields
\begin{equation}
\label{eq:expec-den}
\condexpecunder{g_u\indic{\xi_u\notin S}+ g_{h,u}\indic{\xi_u\in S}}{\ell\notin S}{s} = \left(1-\frac{s}{d-1}\right) g_u + \frac{s}{d-1} g_{h,u}
\, .
\end{equation}
Injecting Eqs.~\eqref{eq:expec-numerator} and~\eqref{eq:expec-den} into Eq.~\eqref{eq:control-single-term}, we obtain 
\[
\condexpecunder{\frac{g_tv_t\indic{\xi_t\notin S} + g_{h,t}v_{h,t}\indic{\xi_t\in S}}{\sum_{u=1}^{\tmax} \left\{ g_u\indic{\xi_u\notin S} + g_{h,u}\indic{\xi_u\in S}\right\}}}{\ell\notin S}{s} = 
\frac{g_tv_t}{ \left(1-\frac{s}{d-1}\right)\sum_u g_u + \frac{s}{d-1}\sum_u g_{h,u} + \frac{s}{d-1}(g_t-g_{h,t})}
+ \bigo{\tmax^{-3/2}}
\, .
\]
Under our assumptions, $\frac{s}{d-1}(g_t-g_{h,t})$ is $\bigo{1}$, whereas the remainder of the denominator is of order at least $\tmax$. 
We deduce that 
\begin{equation}
\label{eq:single-term-approx-1}
\condexpecunder{\frac{g_tv_t\indic{\xi_t\notin S} + g_{h,t}v_{h,t}\indic{\xi_t\in S}}{\sum_{u=1}^{\tmax} \left\{ g_u\indic{\xi_u\notin S} + g_{h,u}\indic{\xi_u\in S}\right\}}}{\ell\notin S}{s} = \frac{g_tv_t}{ \left(1-\frac{s}{d-1}\right)\sum_u g_u + \frac{s}{d-1}\sum_u g_{h,u}}
+ \bigo{\tmax^{-3/2}}
\, .
\end{equation}
We deduce the result coming back to the initial decomposition. 
\qed 


\section{Technical results}
\label{sec:appendix:technical}

\begin{lemma}[Expected ratio]
\label{lemma:expected-frac}
Let $X$ and $Y$ be two random variables with finite variance. 
Assume that there exist two positive constants $c$ and $C$ such that $\abs{X}\leq C$ and $cn \leq Y \leq Cn$ a.s. 
Then 
\[
\abs{\expec{\frac{X}{Y}} - \frac{\expec{X}}{\expec{Y}}} \leq 
\frac{C\var{Y}}{c^3n^3} + \frac{C^2\sqrt{\var{Y}}}{c^2n^2}
\, . 
\]
\end{lemma}


\begin{proof}
Set $\psi : \Reals^2_+ \to \Reals$ defined as $\psi(x,y)\defeq x/y$. 
Multivariate Taylor expansion at order $1$ with integral remainder for an arbitrary $(x_0,y_0)\in\Reals_+^2$ yields
\begin{align*}
\psi(x,y) &= \psi(x_0,y_0) + (x-x_0)\partial_x \psi (x_0,y_0) + (y-y_0)\partial_y \psi(x_0,y_0) \\
&+ \frac{2}{2!}(x-x_0)^2 \int_0^1 (1-t) \partial_{xx} \psi ((x_0,y_0) + t (x-x_0,y-y_0))\Diff t \\
&+ \frac{2}{1!1!}(x-x_0)(y-y_0) \int_0^1 (1-t) \partial_{xy} \psi ((x_0,y_0) + t (x-x_0,y-y_0))\Diff t \\
&+ \frac{2}{2!}(y-y_0)^2 \int_0^1 (1-t) \partial_{yy} \psi ((x_0,y_0) + t (x-x_0,y-y_0))\Diff t 
\, .
\end{align*}
Let us focus on the remainder (the last three lines of the previous display). 
Since $\partial_{xx}\psi = 0$, $\partial_{xy}\psi = -1/y^2$, and $\partial_{yy}\psi = 2x/y^3$, we are left with
\[
-2(x-x_0)(y-y_0)\int_0^1 \frac{(1-t)\Diff t}{(y_0 + t(y-y_0))^2} + 2(y-y_0)^2 \int_0^1 \frac{(1-t)(x_0 + t(x-x_0))\Diff t}{(y_0+t(y-y_0))^3}
\, ,
\]
that is,
\[
(y-y_0)\cdot \int_0^1 (1-t) \frac{-2(x-x_0)(y_0+t(y-y_0))+2(y-y_0)(x_0+t(x-x_0))}{(y_0+t(y-y_0))^3} \Diff t
\, .
\]
One can actually compute this integral, which is
\[
(y-y_0)\cdot \frac{x_0y - xy_0}{yy_0^2} = (y-y_0)\cdot \frac{x_0(y-y_0) -(x-x_0)y_0}{yy_0^2} 
\, .
\]
Going back to the original expansion, we have proved that
\begin{equation}
\label{eq:key-taylor-frac}
\psi(x,y) = \psi(x_0,y_0) + (x-x_0)\partial_x \psi (x_0,y_0) + (y-y_0) \partial_y \psi(x_0,y_0) + (y-y_0)\cdot \frac{x_0(y-y_0) -(x-x_0)y_0}{yy_0^2} 
\, .
\end{equation}
Let us now use Eq.~\eqref{eq:key-taylor-frac} with $x=X$, $y=Y$, $x_0=\expec{X}$, and $y_0=\expec{Y}$, and then take expectation on both sides.
We see that the linear term vanishes, and we are left 
\[
\expec{\frac{X}{Y}} - \frac{\expec{X}}{\expec{Y}} = \expec{(Y-\expec{Y})\cdot \frac{\expec{X}(Y-\expec{Y}) -(X-\expec{X})\expec{Y}}{Y\expec{Y}^2} }
\, .
\]
In absolute value, the last display is smaller than
\[
\frac{C\var{Y}}{c^3n^3} + \frac{C\sqrt{\var{X}\var{Y}}}{c^2n^2}
\, ,
\]
and we deduce the result using Popoviciu's inequality to bound the variance of $X$. 
\end{proof}

We conclude with a technical result used in approximating an integral appearing in the proof of Theorem~\ref{th:lime-meets-attention}. 

\begin{lemma}[Integral approximation]
\label{lemma:integral-approximation}
Let $a>0$. 
Then
\[
\int_0^1 \frac{x\Diff x}{1+ax} = \frac{a-\log (a+1)}{a^2} = \frac{1}{2} - \frac{a}{3} + \bigo{a^2}
\, .
\]
\end{lemma}

\begin{proof}
Taylor expansion. 
\end{proof}


\subsection{Conditional variance computations}

To use Lemma~\ref{lemma:expected-frac} in a meaningful way, the variance of the denominator needs to be controlled.
We show that this is the case with this next result.

\begin{lemma}[Conditional variance computation]
\label{lemma:conditional-variance-fixed}
Let $a_i$ and $b_i$ be two sequences of positive numbers for $i\in [n]$. 
We set $H_S$ as before. 
Let $\ell\in [n]$. 
Then, for all $s\in [n-1]$,
\[
\condexpecunder{H_S}{\ell \notin S}{s} = \frac{n-1-s}{n-1} \sum_i a_i + \frac{s}{n-1}\sum_i b_i + \frac{s}{n-1}(a_\ell - b_\ell)
\, ,
\]
and
\[
\condvarunder{H_S}{s}{\ell \notin S} = \frac{ns(n-s-1)}{(n-1)(n-2)} \left[ \empvar{a-b} - \frac{1}{n-1}\left(a_{\ell}-b_\ell - (\abar - \bbar)\right)^2\right]
\, ,
\]
where $\abar$ (resp. $\bbar$) denote the empirical mean of $a$ (resp. $b$).
\end{lemma}

Lemma~\ref{lemma:conditional-variance-fixed} is somewhat remarkable, connecting the variance of the random sum underpinning our problems to the empirical variance of the coefficients. 
In particular, if we assume that the variance of the summands is $\bigo{1}$, then $\condvarunder{H_S}{s}{\ell \notin S}=\bigo{n}$. 

\begin{proof}
We first write
\begin{align*}
\condexpecunder{H_S}{\ell\notin S}{s} &= \condexpecunder{\sum_i \left\{ a_i\indic{i\notin S} + b_i \indic{i\in S}\right\}}{\ell\notin S}{s} \\
&= \sum_i a_i \condprobaunder{i\notin S}{s}{\ell \notin S} + \sum_i b_i \condprobaunder{i\in S}{s}{\ell \notin S}
\, .
\end{align*}
Taking special care of the case $i=\ell$ in the previous display and using Lemma~\ref{lemma:exact-proba-computations-conditional}, we obtain
\[
\condexpecunder{H_S}{\ell\notin S}{s} = \frac{n-1-s}{n-1}\sum_{i\neq \ell} a_i + \frac{s}{n-1} \sum_{i\neq \ell} b_i + a_{\ell} 
\, .
\]
Rearranging this expression yields the fist statement of the lemma. 
We now turn to the variance computation. 
Without loss of generality, we can assume that $b=0$ since
\[
H_S = \sum_i \left\{a_i\indic{i\notin S} + b_i\indic{i\in S}\right\} = \sum_i (a_i - b_i)\indic{i\notin S}  + \sum_i b_i
\, .
\]
Moreover, we notice that 
\[
\sum_i (a_i+\lambda)\indic{i\notin S} = \sum_i a_i \indic{i\notin S} + \lambda \sum_i \indic{i\notin S} = \sum_i a_i \indic{i\notin S} + (n-s)\lambda
\, .
\]
Thus, without loss of generality, we can assume that $\sum_i a_i=0$. 
Under this assumption, the expectation is simply
\[
\condexpecunder{H_S}{\ell\notin S}{s} = \frac{s}{n-1} a_{\ell}
\, .
\]
We then compute the second raw moment:
\begin{align*}
\condexpecunder{H_S^2}{\ell\notin S}{s} &= \condexpecunder{\left(\sum_i a_i \indic{i\notin S}\right)^2}{\ell\notin S}{s} \\
&= \sum_i a_i^2 \condprobaunder{i\notin S}{s}{\ell\notin S} + 2\sum_{i < j} a_ia_j \condprobaunder{i\notin S, j\notin S}{s}{\ell\notin S}
\, .
\end{align*}
As before, we have to take care of equality cases. 
Using Lemma~\ref{lemma:exact-proba-computations-conditional} and our assumption that $\sum_i a_i=0$, we find
\begin{align*}
\condexpecunder{H_S^2}{\ell\notin S}{s} &= \left(\sum_{i\neq \ell} a_i^2\right) \frac{n-s-1}{n-1} + a_\ell^2 + \left(2\sum_{\substack{i < j \\ i,j\neq \ell}} a_i a_j\right) \frac{(n-s-1)(n-s-2)}{(n-1)(n-2)} \\
&+ \left( 2a_{\ell}\sum_{i\neq \ell}a_i\right)\frac{n-1-s}{n-1} \\
&= \left(\sum_i a_i^2 \right) \frac{n-s-1}{n-1} + \left(2\sum_{i < j} a_ia_j\right) \frac{(n-s-1)(n-s-2)}{(n-1)(n-2)} - a_{\ell}^2 \frac{s(n-2s)}{(n-1)(n-2)} \\
&= \frac{s(n-s-1)}{n-1)(n-2)} \sum_i a_i^2 - \frac{s(n-2s)}{(n-1)(n-2)}  a_{\ell}^2
\, ,
\end{align*} 
since $2\sum_{i<j} a_ia_j = -\sum_i a_i^2$. 
Putting everything together, we obtain
\[
\condvarunder{H_S}{s}{\ell \notin S} = \frac{ns(n-s-1)}{(n-1)(n-2)} \left[ \empvar{a} - \frac{1}{n-1}a_{\ell}^2\right]
\, ,
\]
from which we deduce the result. 
\end{proof}


\subsection{Probability computations}


\begin{lemma}[Exact expressions]
\label{lemma:exact-proba-computations}
Let $a,b,c$ be distinct elements of $[n]$. 
Then
\begin{equation}
\begin{cases}
\probaunder{a\notin S}{s} = \frac{n-s}{n} \\
\probaunder{a\notin S, b\notin S}{s} = \frac{(n-s)(n-1-s)}{n(n-1)} \\
\probaunder{a\notin S,b\in S}{s} = \frac{s(n-s)}{n(n-1)} \\
\probaunder{a\notin S, b\notin S,c\notin S}{s} = \frac{(n-s)(n-s-1)(n-s-2)}{n(n-1)(n-2)} \\
\probaunder{a\notin S, b\in S,c\notin S}{s} = \frac{s(n-s-1)(n-s)}{n(n-1)(n-2)}
\end{cases}
\end{equation}
\end{lemma}

\begin{proof}
Similar to Lemma~4 in \citet{mardaoui_garreau_2021}.
\end{proof}

\begin{lemma}[Exact expressions, conditional]
\label{lemma:exact-proba-computations-conditional}
Let $a,b,\ell$ be distinct elements of $[n]$. 
Then
\begin{equation}
	\begin{cases}
		\condprobaunder{a\notin S}{s}{\ell\notin S} = \frac{n-1-s}{n-1} \\
		\condprobaunder{a\in S}{s}{\ell\notin S} = \frac{s}{n-1} \\
\condprobaunder{a\notin S,b\notin S}{s}{\ell \notin S} =\frac{(n-s-1)(n-s-2)}{(n-1)(n-2)} 
	\end{cases}
\end{equation}
\end{lemma}

\begin{proof}
Let us prove the first statement, the other ones are similar. 
By Bayes formula,
\[
\condprobaunder{a\notin S}{s}{\ell\notin S} = \frac{\probaunder{a\notin S,\ell \notin S}{s}}{\probaunder{\ell\notin S}{s}}
\, .
\]
We now simply use Lemma~\ref{lemma:exact-proba-computations}. 
\end{proof}




\section{Experiments on multi-layer architecture}
\label{sec:appendix-multi-layer}

\begin{figure}
    \centering
    \includegraphics[scale=0.19]{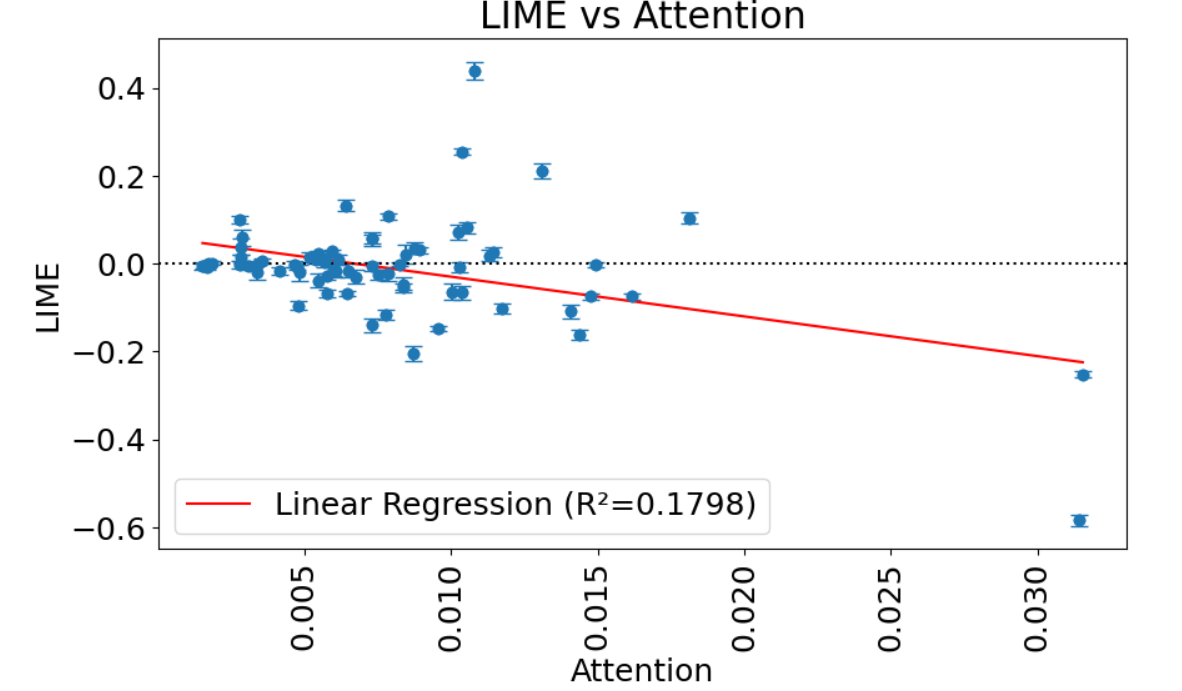}
    \caption{\label{fig:multi-layer}Relation between LIME weights and attention for a $6$-layer $6$-head attention-based classifier. 
    Attention weights correspond to the average attention over the $6$ heads of the first layer. 
    Error bars represent the standard deviation of LIME weights on $5$ repetition.}
\end{figure}
We have conducted numerical experiments on a multi-head, multi-layer architecture. 
We trained a classifier with $6$ layers and $6$ heads on the IMDb dataset (refer to Appendix \ref{sec:experiment}), achieving an accuracy of $82.22\%$.  
Our interest lies in exploring the relationship between LIME and the attention weights. 
We measured the correlation between LIME coefficients and the $\alpha$-avg (refer to Eq. \eqref{eq:lime-meets-attention}) for the first layer. 
An illustration is available at Figure \ref{fig:multi-layer}, where the document corresponds to Figure \ref{fig:accuracy-lime-meets-attention} of our paper. 
The Pearson's correlation in this case is $\rho=-0.424$. 
We attribute the negative sign to the document being classified as negative (as in Figure \ref{fig:accuracy-lime-meets-attention}). 
Attention weights consistently fall within the range of $[0,1]$. 
Considering the absolute values, the rankings of the two explanations are closely aligned, and the correlation is $\rho=0.672$. 
Although we cannot explicitly state the dependency for multi-layers as in Eq. \eqref{eq:lime-meets-attention}, our experiments suggest a significant relationship. 
We conclude that the attention weights are interconnected with LIME coefficients, which adapt more effectively to the model. 
We are currently conducting broader experiments and will incorporate their results and subsequent discussions into the manuscript. 

\section{Experiments}
\label{sec:experiment}
In this section, we report technical details for the model and the experiments.  
Any of the experiments presented in this paper have been performed on a \texttt{PyTorch} implementation of the model presented in Section \ref{sec:the-model} and ran on one GPU Nvidia A100. 

\paragraph{Code.}
The full code is available at \url{https://github.com/gianluigilopardo/attention_meets_xai}. 

\paragraph{Model.}
The model parameters were set as follows: $\tmax=256$, $d_e=128$, $\datt=64$, $\dout=64$. 

\paragraph{Dataset.}
We trained the model on the \texttt{IMDB} dataset \citep{maas_et_al_2011}, which was preprocessed using standard tokenization and padding techniques. 
The dataset was split into training, validation, and test sets with sizes of $20,000$, $5,000$, and $25,000$ samples, respectively.

\paragraph{Training.}
The model was trained for $10$ epochs using a batch size of 16. 
We employed the \texttt{AdamW} optimizer with a learning rate of $0.0001$ and used cross-entropy loss as the optimization objective.

\end{document}